\documentclass[11pt]{article}
\pdfoutput=1

\usepackage{microtype}
\usepackage{graphicx}
\usepackage{subfigure}
\usepackage{booktabs} 
\usepackage{bm}
\usepackage{diagbox}
\usepackage{algorithm,algorithmic}
\usepackage{amsmath,amsfonts,amsthm}
\usepackage{thm-restate}
\usepackage[algo2e,ruled,vlined]{algorithm2e}

\usepackage{geometry}
\usepackage{amssymb}
\usepackage{tikz}
\usepackage{url}
\usepackage{setspace}
\usepackage[pdftex,bookmarksnumbered,bookmarksopen,
colorlinks,citecolor=blue,linkcolor=blue,urlcolor=blue]{hyperref}
\usepackage{framed}
\usepackage{xcolor}
\usepackage{soul}
\usepackage{longtable}

\usepackage{times}
\usepackage{enumitem}
\usepackage{varwidth}
\usepackage{graphicx}
\usepackage{wrapfig}
\usepackage{enumerate}
\usepackage{caption}

\usepackage{amssymb}
\usepackage{bbm}
\usepackage{graphicx}
\usepackage{url}
\usepackage{setspace}
\usepackage{framed}
\usepackage{xcolor}
\usepackage{soul}
\usepackage{longtable}

\usepackage{times}
\usepackage{enumitem}
\usepackage{varwidth}
\usepackage{graphicx}
\usepackage{wrapfig}
\usepackage{enumerate}

\usepackage[utf8]{inputenc} 
\usepackage[T1]{fontenc}    
\usepackage{url}            
\usepackage{booktabs}       
\usepackage{amsfonts}       
\usepackage{nicefrac}       
\usepackage{microtype}      

\usepackage{graphicx}
\usepackage{appendix}
\usepackage{mdwlist}
\usepackage{xspace}
\usepackage{color}
\usepackage{mathrsfs}

\usepackage{booktabs}
\usepackage{comment}

\usepackage{multirow}

\newcommand{\sign}{\textup{\textsf{sign}}}
\newcommand{\sgn}{\textup{\textsf{sgn}}}
\newcommand{\diag}{\textsf{Diag}}

\newcommand{\pdim}{\textsc{Pdim}}

\newcommand{\Appendix}[1]{the full version for}

\newtheorem{theorem}{Theorem}[section]
\newtheorem{lemma}[theorem]{Lemma}

\newtheorem{corollary}[theorem]{Corollary}
\newtheorem{remark}{Remark}

\newtheorem{assumption}{Assumption}
\newtheorem{definition}{Definition}

\renewcommand{\b}{\mathbf{b}}

\newcommand{\x}{\bm{x}}

\newcommand{\cE}{\mathcal{E}}

\newcommand{\Hr}{\mathcal{H}_{\text{Ridge}}}
\newcommand{\Hl}{\mathcal{H}_{\text{LASSO}}}
\newcommand{\Hen}{\mathcal{H}_{\text{EN}}}

\newcommand{\bI}{\mathbb{I}}

\newcommand{\R}{\mathbb{R}}

\newcommand{\Xv}{X_{\text{val}}}
\newcommand{\yv}{y_{\text{val}}}

\renewcommand{\comment}[1]{}

\newcommand{\blue}[1]{{\color{black}#1}}

\newcommand{\cA}{\mathcal{A}}

\newcommand{\cC}{\mathcal{C}}
\newcommand{\cD}{\mathcal{D}}
\newcommand{\cF}{\mathcal{F}}
\newcommand{\cG}{\mathcal{G}}
\newcommand{\cH}{\mathcal{H}}

\newcommand{\cL}{\mathcal{L}}

\newcommand{\cN}{\mathcal{N}}

\newcommand{\bP}{\text{Pr}}
\newcommand{\cP}{\mathcal{P}}

\newcommand{\cR}{\mathcal{R}}

\newcommand{\cX}{\mathcal{X}}

\newcommand{\bbE}{\mathbb{E}}

\newcommand{\norm}[1]{\left\lVert#1\right\rVert}

\definecolor{colorY}{rgb}{0.7 , 0.7 , 0.2}

\DeclareMathOperator*{\argmax}{argmax}
\DeclareMathOperator*{\argmin}{argmin}

\newenvironment{proofoutline}{\noindent{\emph{Proof Sketch. }}}{\hfill$\square$\medskip}

\title{\bf Provably tuning the ElasticNet across instances}

\author{ {Maria-Florina Balcan}\qquad {Mikhail Khodak}\qquad   {Dravyansh Sharma}\qquad   {Ameet Talwalkar}\\ {Carnegie Mellon University\footnote{Correspondence: {\tt ninamf@cs.cmu.edu,  khodak@cmu.edu, dravyans@cs.cmu.edu, talwalkar@cmu.edu}}}}

\date{}

\begin{document}
\maketitle

\begin{abstract}
    An important unresolved challenge in the theory of regularization is to set the regularization coefficients  of popular techniques like the ElasticNet with general provable guarantees.
	We consider the problem of tuning the regularization parameters of Ridge regression, LASSO, and the ElasticNet across multiple problem instances, a setting that encompasses both cross-validation and multi-task hyperparameter optimization. We obtain a novel structural result for the ElasticNet which characterizes the loss as a function of the tuning parameters as a piecewise-rational function with algebraic boundaries.
	We use this to bound the structural complexity of the regularized loss functions and show generalization guarantees for tuning the ElasticNet regression coefficients in the statistical setting.
	We also consider the more challenging online learning setting, where we show vanishing average expected regret relative to the optimal parameter pair. We further extend our results to tuning classification algorithms obtained by thresholding regression fits regularized by Ridge, LASSO, or ElasticNet.
	Our results are the first general learning-theoretic guarantees for this important class of problems that avoid strong assumptions on the data distribution. Furthermore, our guarantees hold for both validation and popular information criterion objectives. 

\end{abstract}

\section{Introduction}\label{sec:intro}
Ridge regression \cite{hoerl1970ridge,tikonov1977solutions}, LASSO \cite{tibshirani1996regression}, and their generalization the ElasticNet \cite{hastie2009elements} are among the most popular algorithms in machine learning and statistics, with applications to linear classification, regression, data analysis, and feature selection \cite{chambers1992linear,zhao2007stagewise,hastie2009elements,dobriban2018high,fernandez2019extensive}.
Given a supervised dataset $(X,y)\in\R^{m\times p}\times\R^m$ with $m$ datapoints and $p$ features, these algorithms compute the linear predictor 
\begin{equation}\label{eq:elasticnet}
\hat\beta_{\lambda_1,\lambda_2}^{(X,y)}
=\argmin_{\beta\in\R^p}\|y-X\beta\|_2^2+\lambda_1\|\beta\|_1+\lambda_2\|\beta\|_2^2.
\end{equation}
Here $\lambda_1,\lambda_2\ge0$ are {\em regularization coefficients} constraining the $\ell_1$ and $\ell_2$ norms, respectively, of the model $\beta$.
For general $\lambda_1$ and $\lambda_2$ the above algorithm is the ElasticNet, while setting $\lambda_1=0$ recovers Ridge and setting $\lambda_2=0$ recovers LASSO.

These coefficients play a crucial role across fields:
in machine learning controlling the norm of $\beta$ implies provable generalization guarantees and prevent over-fitting in practice~\cite{mohri2018fml}, in data analysis their combined use yields parsimonious and interpretable models~\cite{hastie2009elements}, and in Bayesian statistics they correspond to imposing specific priors on $\beta$~\cite{murphy2012machine,li2010bayesian}.
In practice, $\lambda_2$ regularizes $\beta$ by uniformly shrinking all coefficients, while $\lambda_1$ encourages the model vector to be sparse.
This means that while they do yield learning-theoretic and statistical benefits, setting them to be too high will cause models to under-fit the data.
The question of how to set the regularization coefficients becomes even more unclear in the case of the ElasticNet, as one must juggle trade-offs between sparsity, feature correlation, and bias when setting both $\lambda_1$ and $\lambda_2$ simultaneously. As a result, there has been intense empirical and theoretical effort devoted to automatically tuning these parameters. Yet the state-of-the-art is quite unsatisfactory: proposed work consists of either heuristics without formal guarantees~\cite{gibbons1981simulation,kirkland2015lasso}, approaches that optimize over a finite grid or random set instead of the full continuous domain~\cite{chichignoud2016practical}, or analyses that involve very strong theoretical assumptions~\cite{zhang2009some}.

In this work, we study a variant on the above well-established and intensely studied formulation. The key distinction is that instead of a single dataset $(X,y)$, we consider a collection of datasets or instances of the same underlying regression problem $(X^{(i)},y^{(i)})$ and would like to learn a pair $(\lambda_1,\lambda_2)$ that selects a model in equation \eqref{eq:elasticnet} that has low loss on a validation dataset. This can be useful to model practical settings, for example where new supervised data is obtained several times or where the set of features may change frequently \cite{dhurandhar2014efficient}. We do not require all examples across datasets to be i.i.d.\ draws from the same data distribution, and can capture more general data generation scenarios like cross-validation and multi-task learning \cite{zhang2021survey}. Despite these advantages, we remark that our problem formulation is quite different from the standard single dataset setting, \blue{where all examples in the dataset are typically assumed to be drawn independently from the same distribution}. Our formulation treats the selection of regularization coefficients as {\it data-driven algorithm design}, which is often used to study combinatorial problems \cite{gupta2017pac,balcan2020data}.\looseness-1

Our main contribution is a new structural result for the ElasticNet Regression problem, which implies generalization guarantees for selecting ElasticNet Regression coefficients in the multiple-instance setting. In particular, Ridge and LASSO regressions are special cases. We extend our results to obtain low regret in the online learning setting, and to tuning related linear classification algorithms.
In summary, we make the following key contributions:
\begin{itemize}[leftmargin=*]
	\item We formulate the problem of tuning the ElasticNet as a question of learning $\lambda_1$ and $\lambda_2$ simultaneously across multiple problem instances, either generated statistically or coming online. Our formulation captures relevant settings like cross-validation and multi-task learning. 
    \item We provide a novel structural result (Theorem \ref{lem:en-structure}) that characterizes the loss of the ElasticNet fit. We show that the hyperparameter space can be partitioned by polynomial curves of bounded degrees into pieces where the loss is a bivariate rational function.
    The result holds for both the usual ElasticNet validation objective and when it is augmented with information criteria like the AIC or BIC.
	\item An important consequence of our structural result is a bound on the pseudo-dimension (Definition \ref{def:pd}) for the loss function class, which yields strong generalization bounds for tuning $\lambda_1$ and $\lambda_2$ simultaneously in the statistical learning setting (Theorem \ref{thm:en-sc}). Informally, for ElasticNet regression problems with at most $p$ parameters, for any problem distribution $\cD$, we show that $O\left(\frac{1}{\epsilon^2}(p^2\log\frac{1}{\epsilon}+\log\frac{1}{\delta})\right)$ problems (datasets) are sufficient to learn an $\epsilon$-approximation to the best $(\lambda_1,\lambda_2)$, with probability at least $1-\delta$. 
	\item In the online setting, we show under very mild data assumptions---much weaker than prior work---that the problem satisfies a dispersion condition \cite{balcan2018dispersion,sharma2020learning}.
	As a result we can tune all parameters across a sequence of instances appearing online and obtain vanishing regret relative to the optimal parameter in hindsight over the sequence (Theorem \ref{thm:en-regression-dispersion}) at the rate $\Tilde{O}(1/\sqrt{T})$\footnote{The soft-O notation is used to emphasize dependence on $T$, and suppresses other factors as well as logarithmic terms.} wrt the length $T$ of the sequence.
    \item We show how to extend our results to regularized classifiers that perform thresholding on Ridge, LASSO or ElasticNet regression estimates, again providing strong generalization and online learning guarantees (Theorems \ref{thm:ridge-pdim}, \ref{thm:ridge-dispersion}).
\end{itemize}

We include a couple of remarks to emphasize the generality and significance of our results.
First, in our multiple-instance formulation the different problem instances need not have the same number of examples, or even the same set of features. This allows us to handle practical scenarios where the set of features changes across datasets, and we can learn parameters that work well on average across multiple different but related regression tasks. Second, by generating problem instances iid from a fixed (training + validation) dataset, we can obtain iterations (training/validation splits) of popular cross-validation techniques (including the popular leave-one-out and Monte Carlo CV) and our result implies that $\Tilde{O}(p^2/\epsilon^2)$ iterations are enough to determine an ElasticNet parameter $\hat{\lambda}$ with loss within $\epsilon$ (with high probability) of the optimal parameter $\lambda^*$ over the distribution induced by the cross-validation splits.

{\bf Key challenges and insights}. 
A major challenge in learning the ElasticNet parameters is that the variation of the solution path as a function of $\lambda_2$ is hard to characterize. Indeed the original ElasticNet paper \cite{zou2005regularization} suggests using the heuristic of grid search to learn a good $\lambda_2$, even though $\lambda_1$ may be exactly optimized by computing full solution paths (for each $\lambda_2$). We approach this indirectly by utilizing a characterization of the LASSO solution by \cite{tibshirani2013lasso}, which is based on the KKT (Karush–Kuhn–Tucker) optimality conditions, to arrive at a precise piecewise structure for the problem. In more detail, we use these conditions to come up with a set of algebraic curves (polynomial equations in $\lambda_1$ and $\lambda_2$) of bounded degrees, such that \blue{the set of possible discontinuities lie along these curves, and the loss function behaves well (a bounded-degree rational function) in each piece of the partition of the parameter domain induced by these curves}. This characterization is crucial in establishing a bound on the structural complexity needed to provide strong generalization guarantees. We further show additional structure on these algebraic curves that (roughly speaking) imply that the curves do not concentrate in any region of the domain, allowing us to use the powerful recipe of \cite{dick2020semi} for online learning.\looseness-1

\subsection{Related work}

Model selection for Ridge regression, LASSO or ElasticNet is typically done by selecting the regularization parameter $\lambda$ that works well for given data, although some parameter-free techniques for variable selection have been recently proposed \cite{lederer2015don}. Choosing `optimal’ parameters for tuning the regularization has been a subject of extensive theoretical and applied research. Much of this effort is heuristic \cite{gibbons1981simulation,kirkland2015lasso} or focused on developing tuning objectives beyond validation accuracy like AIC or BIC \cite{akaike1974new,schwarz1978estimating} without providing procedures for provably optimizing them.
The standard approach given a tuning objective is to optimize it over a grid or random set of parameters, for which there are guarantees \cite{chichignoud2016practical}, but this does not ensure optimality over the entire continuous tuning domain, especially since objectives such as 0-1 validation error or information criteria can have many discontinuities. Selecting a grid that is too fine or too coarse can result in either very inefficient or highly inaccurate estimates (respectively) for good parameters.
Other guarantees make strong assumptions on the data distribution such as sub-Gaussian noise \cite{zhang2009some,chetverikov2021cross} or depend on unknown parameters that are hard to quantify in practice \cite{fan2010selective}. {Recent work has shown asymptotic consistency of cross-validation for ridge regression, even in the limiting case $\lambda_2\rightarrow 0$ which is particularly interesting for the overparameterized regime \cite{hastie2022surprises,patil2021uniform}.}
A successful line of work has focused on efficiently obtaining models for different values of $\lambda_1$ using regularization paths \cite{efron2004least}, but the guarantees are computational rather than learning-theoretic or statistical. In contrast, we provide principled approaches that guarantee near-optimality of selected parameters with high confidence over the entire continuous domain of parameters.

Data-driven algorithm design has proved successful for tuning parameters for a variety of combinatorial problems like clustering, integer programming, auction design and graph-based learning \cite{balcan2019learning,balcan2021sample,balcan2016sample, balcan2021data}. We provide an application of these techniques to parameter tuning in a problem that is not inherently combinatorial by revealing a novel discrete structure. We identify the underlying piecewise structure of the ElasticNet loss function which is extremely effective in establishing learning-theoretic guarantees \cite{balcan2021much}. To exploit this piecewise structure, we analyze the learning-theoretic complexity of rational algebraic function classes and infer generalization guarantees. \blue{Follow-up work \cite{balcan2023new} improves on our generalzation guarantees and extends the results to regularized logistic regression.} We also employ and extend general tools and techniques for online data-driven learning from \cite{dick2020semi,balcan2021data} to rational functions in order to prove our online learning guarantees for regularization coefficient tuning.\looseness-1

\section{Preliminaries and a Key Structural Result}\label{sec:prelim}

Given data $(X,y)$ with $X\in\R^{m\times p}$ and $y\in\R^{m}$, consisting of $m$ labeled examples with $p$ features, we seek estimators $\beta\in\R^p$ which minimize the regularized loss. Popular regularization methods like LASSO and ElasticNet can be expressed as computing the solution of an optimization problem given by \looseness-1

\[\hat{\beta}_{\lambda,f}^{(X,y)} \in \argmin_{\beta\in\R^p}\norm{y-X\beta}_2^2+\langle\lambda ,f(\beta)\rangle,\]

\noindent where $f:\R^p\rightarrow\R_{\ge 0}^d$ gives the regularization penalty for estimator $\beta$, $\lambda\in \R_{\ge 0}^d$ is the regularization parameter, and $d$ is the number of regularization parameters. $d=1$ for Ridge and LASSO, and $d=2$ for the ElasticNet. Setting $f=f_2$ with $f_2(\beta)=\norm{\beta}_2^2$ yields  Ridge regression, and setting  $f(\beta)=f_1(\beta):=\norm{\beta}_1$ corresponds to LASSO. Also using $f_{\text{EN}}(\beta):=(f_1(\beta),f_2(\beta))$ gives the ElasticNet with regularization parameter $\lambda=(\lambda_1,\lambda_2)$. Note that we use the same $\lambda$ (with some notational overloading) to denote the regularization parameters for ridge, LASSO, or ElasticNet. We write $\hat{\beta}_{\lambda,f}^{(X,y)}$ as simply $\hat{\beta}_{\lambda,f}$ when the dataset $(X,y)$ is clear from context. On any instance $x\in\R^p$ from the feature space, the prediction of the regularized estimator is given by the dot product $\langle x, \hat{\beta}_{\lambda,f}\rangle$. The average squared loss over a dataset $(X',y')$ with $X'\in\R^{m'\times p}$ and $y'\in \R^{m'}$ is given by 
$$l_r(\hat{\beta}_{\lambda,f}, (X',y')) = \frac{1}{m'}\norm{y'-X'\hat{\beta}_{\lambda,f}}_2^2.$$ By setting $(X',y')$ to be the training data $(X,y)$, we get the training loss $l_r(\hat{\beta}_{\lambda,f}, (X,y))$. We use $(\Xv,\yv)$ to denote a validation split.

{\it Distributional and Online Settings.} In the {\it distributional or statistical} setting, we receive a collection of $n$ instances of the regression problem $$P^{(i)}=(X^{(i)},y^{(i)},\Xv^{(i)},\yv^{(i)})\in\cR_{m_i,p_i,m_i'}:=\R^{m_i\times p_i}\times\R^{m_i}\times\R^{m'_i\times p_i}\times\R^{m'_i},$$ for $i\in[n]$ generated i.i.d.\ from some problem distribution $\cD$. The problems are in the problem space given by $\Pi_{m,p}=\bigcup_{m_1\ge 0,m_2\le m, p_1\le p}\cR_{m_1,p_1,m_2}$ (note that the problem distribution $\cD$ is over $\Pi_{m,p}$). 
On any given instance $P^{(i)}$ the loss is given by the squared loss on the validation set, $\ell_{\text{EN}}(\lambda,P^{(i)})=l_r(\hat{\beta}_{\lambda,f_{\text{EN}}}^{(X^{(i)},y^{(i)})},(\Xv^{(i)},\yv^{(i)}))$. On the other hand, in the {\it online setting}, we receive a sequence of $T$ instances of the ElasticNet regression problem $P^{(i)}=(X^{(i)},y^{(i)},\Xv^{(i)},\yv^{(i)})\in \Pi_{m,p}$ for $i\in[T]$ online.  On any given instance $P^{(i)}$, the online learner is required to select the regularization parameter $\lambda^{(i)}$ without observing $\yv^{(i)}$, and experiences loss  given by $\ell(\lambda^{(i)},P^{(i)})=l_c(\hat{\beta}_{\lambda^{(i)},f_{EN}}^{(X^{(i)},y^{(i)})},(\Xv^{(i)},\yv^{(i)}))$. The goal is to minimize the regret w.r.t.\ choosing the best fixed parameter in hindsight for the same problem sequence, i.e.\ $$R_T=\sum_{i=1}^T\ell(\lambda^{(i)},P^{(i)})-\min_{\lambda}\sum_{i=1}^T\ell(\lambda,P^{(i)}).$$ We also define average regret as $\frac{1}{T}R_T$ and expected regret as $\bbE[R_T]$ where the expectation is over both the randomness of the loss functions and any random coins used by the online algorithm.

Given a class of regularization algorithms $\cA$ parameterized by regularization parameter $\lambda$ over a set of problem instances $\cX$, and given loss function $\ell:\cA\times\cX\rightarrow \R$ which measures the loss of any algorithm in $\cA$ on any fixed problem instance, consider the set of functions $\cH_\cA=\{\ell(A,\cdot)\mid A\in\cA\}$. For example, for the ElasticNet we have $\ell_{\text{EN}}(\lambda,P)=l_r(\hat{\beta}^{(X_P,y_P)}_{\lambda,f_{\text{EN}}}, (X'_P,y'_P))$, where $(X_P,y_P)$ and $(X'_P,y'_P)$ are the training and validation sets associated with problem $P\in\cX$ respectively. Bounding the pseudo-dimension of $\cH_\cA$ gives a bound on the  sample complexity for uniform convergence guarantees, i.e. a bound on the sample size $n$ for which the algorithm $\hat{A}_S\in\cA$ which minimizes the average loss on any sample $S$ of size $n$ drawn i.i.d. from any problem distribution $\cD$ is guaranteed to be near-optimal with high probability \cite{dudley1967sizes}. See Appendix \ref{app:learning-theory} for the relevant classic definitions and results. 
Define the {\it dual class} $\cH^*$ of a set of real-valued functions $\cH\subseteq 2^{\cX}$ as $\cH^*=\{h^*_x:\cH\rightarrow \R\mid x\in \cX\}$ where $h^*_x(h)=h(x)$. In the context of regression problems $\cX$, for each fixed problem instance $x\in\cX$ there is a dual function $h^*_x$ that computes the loss $\ell(A,x)$ for any (primal) function $h_A=\ell(A,\cdot)\in\cH_\cA$. For a function class $\cH$, showing that dual class $\cH^*$ is piecewise-structured in the sense of Definition \ref{def:ps} and bounding the complexity of the duals of boundary and piece functions of $\cH^*$  are useful to understand the learnability of $\cH$ \cite{balcan2021much}.
\begin{definition}[Piecewise structured functions, \cite{balcan2021much}]\label{def:ps}
    A function class $H \subseteq \R^\cX$ that maps a domain $\cX$ to $\R$ is $(F, G, k)$-\emph{piecewise decomposable} for a class $G \subseteq \{0, 1\}^\cX$ of boundary functions and a class $F \subseteq \R^\cX$ of piece functions if the following holds: for every $h \in H$, there are $k$ boundary functions $g_1,\dots,g_k \in G$ and a piece function $f_\mathbf{b} \in F$ for each bit vector $\mathbf{b} \in \{0, 1\}^k$ such that for all $x \in\cX$, $h(x) = f_{\mathbf{b}_x}(x)$ where $\mathbf{b}_x = (g_1(x),\dots,g_k(x))\in \{0, 1\}^k$.
\end{definition}

\noindent{Intuitively, a real-valued function is piecewise-structured if the domain can be divided into pieces by a finite number of boundary functions (say linear or polynomial thresholds) and the function value over each piece is easy to characterize (e.g. constant, linear, polynomial).} To state and understand our structural insights into the ElasticNet problem we will also need the definition of equicorrelation sets, the subset of features with maximum absolute correlation for any fixed $\lambda_1$, useful for characterizing LASSO/ElasticNet solutions. For any subset $\cE\subseteq[p]$ of the features, we define $X_{\cE}=\begin{pmatrix}\dots X_{* i}\dots \end{pmatrix}_{i\in \cE}$ as the $m\times|\cE|$ matrix of columns $X_{* i}$ of $X$ corresponding to indices $i\in\cE$. Similarly $\beta_\cE\in\R^{|\cE|}$ is the subset of estimators in $\beta$ corresponding to indices in $\cE$. We will assume all the feature matrixes $X$ (for training datasets) are in general position (Definition \ref{def:gp}).

\begin{definition}[Equicorrelation sets, \cite{tibshirani2013lasso}] Let $\beta^*\in \argmin_{\beta\in\R^p}\norm{y-X\beta}_2^2+\lambda_1||\beta||_1$. The equicorrelation set corresponding to $\beta^*$, $\cE=\{j\in[p]\mid |\x_j^T(y-X\beta^*)|=\lambda_1\}$, is simply the set of covariates with maximum absolute correlation. We also define the equicorrelation sign vector for $\beta^*$ as $s=\sign(X_{\cE}^T(y-X\beta^*))\in\{\pm 1\}$.
\label{def:ec}
\end{definition}

\noindent \blue{Here $\sign(x)=1$ if $x\ge0$, and $\sign(x)=-1$ otherwise}. Consider the class of algorithms consisting of ElasticNet regressors for different values of $\lambda=(\lambda_1,\lambda_2)\in(0,\infty)\times(0,\infty)$. We assume $\lambda_1>0$ for technical simplicity (cf. \cite{tibshirani2013lasso}). We seek to solve problems of the form $P=(X,y,\Xv,\yv)\in\Pi_{m,p}$, where $(X,y)$ is the training set, $(\Xv,\yv)$ is the validation set with the same set of features, and $m,p$ are upper bounds on the number of examples and features respectively in any dataset. 
Let $\Hen=\{\ell_{\text{EN}}(\lambda,\cdot)\mid \lambda\in(0,\infty)\times(0,\infty) \}$ denote the set of loss functions for the class of algorithms consisting of ElasticNet regressors for different values of $\lambda\in\R^+\times\R^+$. Additionally, we will consider information criterion based loss functions, $\ell_{\text{EN}}^{\textsf{AIC}}(\lambda,P)=\ell_{\text{EN}}(\lambda,P)+2||\hat{\beta}_{\lambda,f_{\text{EN}}}^{(X,y)}||_0$ and $\ell_{\text{EN}}^{\textsf{BIC}}(\lambda,P)=\ell_{\text{EN}}(\lambda,P)+2||\hat{\beta}_{\lambda,f_{\text{EN}}}^{(X,y)}||_0\log m$ \cite{akaike1974new,schwarz1978estimating}. Let $\Hen^{\textsf{AIC}}$ and $\Hen^{\textsf{BIC}}$ denote the corresponding sets of loss functions. These criteria are popularly used to compute the squared loss on the training set, to give alternatives to cross-validation. We do not make any assumption on the relation between training and validation sets in our formulation, so our analysis can capture these settings as well.

\subsection{Piecewise structure of the ElasticNet loss}

We will now establish a piecewise structure of the dual class loss functions (Definition \ref{def:ps}). A key observation is that if the signed equicorrelation set $(\cE,s)$ (i.e. a subset of features $\cE\subseteq[p]$ with the same maximum absolute correlation, assigned a fixed sign pattern $\{-1,+1\}^{|\cE|}$, see Definition \ref{def:ec}) is fixed, then the ElasticNet coefficients may be characterized (Lemma \ref{lem:en-equicorrelation}) and the loss is a fixed rational polynomial piece function of the parameters $\lambda_1,\lambda_2$. We then show the existence of a set of boundary function curves $\cG$, such that any region of the parameter space located on a fixed side of all the curves (more formally, for a fixed sign pattern in Definition \ref{def:ps}) in $\cG$ has the same signed equicorrelation set. The boundary functions are a collection of possible curves at which a covariate may enter or leave the set $\cE$ and correspond to polynomial thresholds. We make repeated use of the following lemma which provides useful properties of the piece functions as well the the boundary functions of the dual class loss functions.

\begin{wrapfigure}{r}{0.44\textwidth}
    
    \vspace*{-0.6cm}
    \centering
 \includegraphics[scale=0.66]{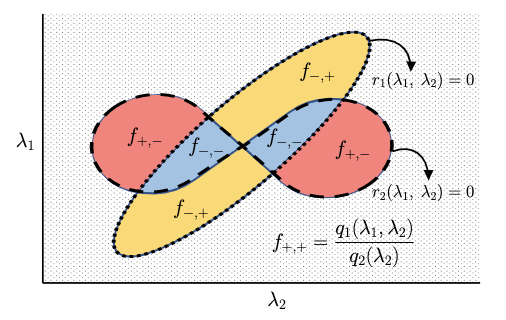}
 
 \caption{An illustration of the piecewise structure of the ElasticNet loss, as a function of the regularization parameters, for a fixed problem instance. Pieces are regions where some bounded degree polynomials ($r_1,r_2$) have a fixed sign pattern (one of $\pm1,\pm1$), and in each piece the loss is a fixed (rational) function.}
    \label{fig:ps}
    
    \vspace*{-0.6cm}
\end{wrapfigure}

\begin{lemma}\label{lem:gram-inv}
Let $A$ be an $r\times s$ matrix. Consider the matrix $B(\lambda)=(A^TA+\lambda I_s)^{-1}$ and $\lambda>0$. \begin{itemize}
    \item[1.] Each entry of $B(\lambda)$ is a rational polynomial $P_{ij}(\lambda)/Q(\lambda)$ for $i,j\in[s]$ with each $P_{ij}$ of degree at most $s-1$, and $Q$ of degree $s$.
    \item[2.] Further, for $i=j$, $P_{ij}$ has degree $s-1$ and leading coefficient 1, and for $i\ne j$ $P_{ij}$ has degree at most $s-2$. Also, $Q(\lambda)$ has leading coefficient $1$. 
\end{itemize}

\end{lemma}

\noindent The proof is straightforward and deferred to Appendix \ref{app:regression}. We will now formally state and prove our key structural result which is needed to establish our generalization and online regret guarantees in Section \ref{sec:regression}.

\begin{theorem}\label{lem:en-structure}
Let $\cL$ be a set of functions $\{l_{\lambda}:\Pi_{m,p}\rightarrow \R_{\ge 0}\mid \lambda\in\R^+\times\R^+$ that map a regression problem instance $P\in \Pi_{m,p}$ to the validation loss $\ell_{\text{EN}}(\lambda,P)$ of ElasticNet trained with regularization parameter $\lambda=(\lambda_1,\lambda_2)$. The dual class $\cL^*$ is $(\cF,\cG,p3^p)$-piecewise decomposable, with $\cF = \{f_q : \cL \rightarrow \R\}$ consisting of rational polynomial functions $f_{q_1,q_2} : l_\lambda\mapsto \frac{q_1(\lambda_1,\lambda_2)}{q_2(\lambda_2)}$, where $q_1,q_2$ have degrees at most $2p$, and $\cG=\{g_r:\cL\rightarrow \{0,1\}\}$ consists of polynomial threshold functions $g_r : u_\lambda \mapsto \mathbb{I}\{r(\lambda_1,\lambda_2)<0\}$, where $r$ is a polynomial of degree 1 in $\lambda_1$ and at most $p$ in $\lambda_2$.
\end{theorem}

\begin{proof}
Let $P=(X,y,\Xv,\yv)\in\Pi_{m,p}$ be a regression problem instance. By using the standard reduction to LASSO \cite{zou2005regularization} and well-known characterization of the LASSO solution in terms of equicorrelation sets, we can characterize the solution $\hat{\beta}_{\lambda,f_{EN}}$ of the Elastic Net as follows (Lemma \ref{lem:en-equicorrelation}):
\[\hat{\beta}_{\lambda,f_{EN}} = (X_{\cE}^TX_{\cE}+\lambda_2I_{|\cE|})^{-1}X_{\cE}^Ty-\lambda_1(X_{\cE}^TX_{\cE}+\lambda_2I_{|\cE|})^{-1}s,\]
for some $\cE\in[p]$ and $s\in\{-1,1\}^p$. Thus for any $\lambda=(\lambda_1,\lambda_2)$, the prediction $\hat{y}$ on any validation example with features $\x\in\R^p$ satisfies (for some $\cE,s\in 2^{[p]}\times \{-1,1\}^p$)
\[\hat{y}_j=\x\hat{\beta}_{\lambda,f_{EN}} = \x(X_{\cE}^TX_{\cE}+\lambda_2I_{|\cE|})^{-1}X_{\cE}^Ty-\lambda_1\x(X_{\cE}^TX_{\cE}+\lambda_2I_{|\cE|})^{-1}s.\]
For any subset $R\subseteq\R^2$, if the signed equicorrelation set $(\cE,s)$ is fixed over $R$, then the above observation, together with Lemma \ref{lem:gram-inv} implies that the loss function $\ell_{\text{EN}}(\lambda,P)$ is a rational function of the form $\frac{q_1(\lambda_1,\lambda_2)}{q_2(\lambda_2)}$, where $q_1$ is a bivariate polynomial with degree at most $2|\cE|$ and $q_2$ is univariate with degree $2|\cE|$.

To show the piecewise structure, we need to demonstrate a set boundary functions $\cG=\{g_1,\dots,g_k\}$ such that for any sign pattern $\b\in\{0,1\}^k$, the signed equicorrelation set $(\cE,s)$  for the region with sign pattern $\b$ is fixed. To this end, based on the observation above, we will consider the conditions (on $\lambda$) under which a covariate may enter or leave the equicorrelation set. We will show that this can happen only at one of a finite number of algebraic curves (with bounded degrees).

{\it Condition for joining $\cE$.} Fix $\cE,s$. Also fix $j\notin \cE$. If covariate $j$ enters the equicorrelation set, the KKT conditions (Lemma \ref{lem:lasso-kkt}) applied to the LASSO problem corresponding to the ElasticNet (Lemma \ref{lem:en-equicorrelation}) imply 
$$(\x_j^*)^T(y^*-X_\cE^*(c_1-c_2\lambda_1^*))=\pm \lambda_1^*,$$
where $c_1= ({X^*_{\cE}}^TX^*_{\cE})^{-1}{X^*_{\cE}}^Ty^*$, $c_2=({X^*_{\cE}}^TX^*_{\cE})^{-1}s$, $X^*=\frac{1}{\sqrt{1+\lambda_2}}\begin{pmatrix}X\\ \sqrt{\lambda_2}I_p\end{pmatrix}$, $y^*=\begin{pmatrix}y\\ 0\end{pmatrix}$, and $\lambda_1^*=\frac{\lambda_1}{\sqrt{1+\lambda_2}}$. Rearranging, and simplifying, we get\looseness-1 
\begin{align*}\lambda_1^*&=\frac{(\x_j^*)^TX^*_\cE({X^*_{\cE}}^TX^*_{\cE})^{-1}(X^*_{\cE})^Ty^*-(\x^*_j)^Ty^*}{(\x_j^*)^TX^*_\cE({X^*_{\cE}}^TX^*_{\cE})^{-1} s\pm1}, \text{or}\\
\lambda_1&=\frac{\x_j^TX_\cE({X_{\cE}}^TX_\cE+\lambda_2 I_{|\cE|})^{-1}{X_{\cE}}^Ty-\x_j^Ty}{\x_j^TX_\cE({X_{\cE}}^TX_\cE+\lambda_2 I_{|\cE|})^{-1} s\pm1}.\end{align*}

\noindent Note that the terms $(\x_j^*)^TX^*_\cE =\x_j^TX_\cE$, $(X^*_{\cE})^Ty^*=X_{\cE}^Ty$, and $(\x^*_j)^Ty^*=\x_j^Ty$ do not depend on $\lambda_1$ or $\lambda_2$ (the $\lambda_2$ terms are zeroed out since $j\notin\cE$). Moreover, $({X^*_{\cE}}^TX^*_{\cE})^{-1}=({X_{\cE}}^TX_\cE+\lambda_2 I_{|\cE|})^{-1}$. Using Lemma \ref{lem:gram-inv}, we get an algebraic curve $r_{j,\cE,s}(\lambda_1,\lambda_2)=0$ with degree 1 in $\lambda_1$ and $|\cE|$ in $\lambda_2$ corresponding to addition of $j\notin\cE$ given $\cE,s$. 

{\it Condition for leaving $\cE$.} Now consider a fixed $j'\in\cE$, given fixed $\cE,s$. The coefficient of $j'$ will be zero for $\lambda_1^*=\frac{(c_1)_{j'}}{(c_2)_{j'}}$, which simplifies to $\lambda_1(({X_{\cE}}^TX_\cE+\lambda_2 I_{|\cE|})^{-1}s)_{j'} = (({X_{\cE}}^TX_\cE+\lambda_2 I_{|\cE|})^{-1}{X_{\cE}}^Ty)_{j'}$. Again by Lemma \ref{lem:gram-inv}, we get an algebraic curve $r_{j',\cE,s}(\lambda_1,\lambda_2)=0$ with degree 1 in $\lambda_1$ and at most $|\cE|$ in $\lambda_2$ corresponding to removal of $j'\in\cE$ given $\cE,s$.

Putting the two together, we get $\sum_{i=0}^p2^i{p\choose i}\left((p-i)+i\right)=p3^p$ algebraic curves of degree 1 in $\lambda_1$ and at most $p$ in $\lambda_2$, across which the signed equicorrelation set may change. These curves characterize the complete set of points $(\lambda_1,\lambda_2)$ at which $(\cE,s)$ may possibly change. Thus by setting these $p3^p$ curves as the set of boundary functions $\cG$, $\cE,s$ is guaranteed to be fixed for each sign pattern, and the corresponding loss takes the rational function form shown above.
\end{proof}

\noindent The exact same piecewise structure can be established for the dual function classes for loss functions $\ell_{\text{EN}}^{\textsf{AIC}}(\lambda,\cdot)$ and $\ell_{\text{EN}}^{\textsf{BIC}}(\lambda,\cdot)$. This is evident from the proof of Theorem \ref{lem:en-structure}, since any dual piece has a fixed equicorrelation set, and therefore $||\beta||_0$ is fixed. Given this piecewise structure, a challenge to learning values of $\lambda$ that minimize the loss function is that the function may not be differentiable (or may even be discontinuous, for the information criteria based losses) at the piece boundaries, making well-known gradient-based (local) optimization techniques inapplicable here. In the following (specifically Algorithm \ref{alg:ddreg}) we will show that techniques from data-driven design may be used to overcome this optimization challenge.

\section{Learning to Regularize the ElasticNet}\label{sec:regression}

We will consider the problem of learning provably good ElasticNet parameters for a given problem domain, from multiple datasets (problem instances) either available as a collection (Section \ref{sec:en-regression-distributional}), or arriving online (Section \ref{sec:en-regression-online}). Our parameter tuning techniques also apply to simpler regression techniques typically used for variable selection, like LARS and LASSO, which are reasonable choices if the features are not multicollinear. Additional proof details for the results in this section are located in Appendix \ref{app:regression}.\looseness-1 

\subsection{Distributional Setting}\label{sec:en-regression-distributional}

Our main result in this section is the following upper bound on the pseudo-dimension of the classes of loss functions for the ElasticNet, which implies that in our distributional setting it is possible to learn near-optimal values of $\lambda$ with polynomially many problem instances.\looseness-1 

\begin{theorem}\label{thm:pdim-en}
$\pdim(\Hen)=O(p^2)$. Further, $\pdim(\Hen^{\textsf{AIC}})=O(p^2)$ and $\pdim(\Hen^{\textsf{BIC}})=O(p^2)$.
\end{theorem}
\begin{proofoutline}
The crucial ingredient is the $(\cF,\cG,p3^p)$-piecewise decomposable structure for the dual class function $\Hen^*$ established in Theorem \ref{lem:en-structure}, \blue{where $\cF$ is a class of bivariate rational functions and $\cG$ consists of polynomial thresholds, both with bounded degrees}. We then bound the complexity of the corresponding dual class functions $\cF^*$ and $\cG^*$, in order to use the following powerful general result due to \cite{balcan2021much} (Theorem \ref{thm:pdim-dual} in the appendix)

\[\pdim(\cH)=O((\pdim(\cF^*)+d_{\cG^*})\log(\pdim(\cF^*)+d_{\cG^*})+d_{\cG^*}\log k).\]

\noindent In more detail, we can bound the pseudo-dimension of the dual class of piece functions $\cF^*$ (a class of bivariate rational functions) by $O(\log p)$ \blue{(Lemma \ref{lem:f-pdim} in the appendix)}, by giving an upper bound of $O(k^3d^3)$ on the number of sign patterns over $\R^2$ induced by $k$ algebraic curves of degree at most $d$. 
We can also bound the VC dimension of the dual class of boundary functions $\cG^*$ (polynomial thresholds in two variates) by $O(p)$ using a standard linearization argument (Lemma \ref{lem:g-vcdim}). Finally, the above result from \cite{balcan2021much} allows us to bound the pseudodimension of $\cH$ by combining the above bounds. 

\blue{\[\pdim(\cH)=O(p\log p+p\log(p3^p))=O(p^2).\]

\noindent The dual classes ${(\Hen^{\textsf{AIC}})}^*$ and ${(\Hen^{\textsf{BIC}})}^*$ also follow the same piecewise decomposable structure given by Theorem \ref{lem:en-structure}. This is because in each piece the equicorrelation set $\cE$, and therefore $||\beta||_0=|\cE|$  is fixed (Lemma \ref{lem:lasso}). 
The above argument implies an identical upper bound on the pseudo-dimensions of $\Hen^{\textsf{AIC}}$ and $\Hen^{\textsf{BIC}}$. See Appendix \ref{app:regression} for further proof details, including the technical lemmas.}
\end{proofoutline}

\noindent 
The upper bound above implies a guarantee on the sample complexity of learning the ElasticNet tuning parameter, using standard learning-theoretic results \cite{anthony1999neural}, \blue{under mild boundedness assumptions on the data and hyperparameter search space. 

\begin{assumption}[Boundedness]
The predicted variable and all feature values are bounded by an absolute constant $R$, i.e. $\max\{||X^{(i)}||_{\infty,\infty},||y^{(i)}||_\infty,||\Xv^{(i)}||_{\infty,\infty},||\yv^{(i)}||_\infty\}\le R$. Furthermore, the regularization coefficients are bounded, $(\lambda_1,\lambda_2)\in[\lambda_{\min},\lambda_{\max}]^2$ for $0<\lambda_{\min}<\lambda_{\max}<\infty$.
\label{ass:boundedness}
\end{assumption}}

\noindent In our setting of learning from multiple problem instances, each sample is a dataset instance, so the sample complexity is simply the number of regression problem instances needed to learn the tuning parameters to any given approximation and confidence level.

\begin{theorem}[Sample complexity of tuning the ElasticNet] Suppose Assumption \ref{ass:boundedness} holds. Let $\cD$ be an arbitary distribution over the problem space $\Pi_{m,p}$. There is an algorithm which given $n=O\left(\frac{\blue{H^2}}{\epsilon^2}(p^2+\log\frac{1}{\delta})\right)$ problem samples drawn from $\cD$, for any $\epsilon>0$ and $\delta\in(0,1)$ and \blue{some constant $H$}, outputs a  regularization parameter $\hat{\lambda}$ for the ElasticNet such that with probability at least $1-\delta$ over the draw of the problem samples, we have that\looseness-1 
\[\Big\lvert\bbE_{P\sim\cD}[\ell_{EN}(\hat{\lambda},P)]-\min_{\lambda}\bbE_{P\sim\cD}[\ell_{EN}(\lambda,P)]\Big\rvert\le \epsilon.\]
\label{thm:en-sc}
\end{theorem}
\begin{proof}
\blue{We use Lemma \ref{lemma:bounded-EN} to conclude that the validation loss is uniformly bounded by some constant $H$ under Assumption \ref{ass:boundedness}.} The result then follows from substituting our result in Theorem \ref{thm:pdim-en} into well-known generalization guarantee for function classes with bounded pseudo-dimensions (Theorem \ref{thm:pdim-generalization}).
\end{proof}

\noindent {\it Discussion and applications.} {Computing the parameters which minimize the loss on the problem samples (aka Empirical Risk Minimization, or ERM) achieves the sample complexity bound in Theorem \ref{thm:en-sc}. Even though we only need polynomially many samples to guarantee the selection of nearly-optimal parameters, it is not clear how to implement the ERM efficiently.} Note that we do not assume the set of features is the same across problem instances, so our approach can handle feature reset i.e. different problem instances can differ in not only the number of examples but also the number of features. Moreover, as a special case application, we consider the commonly used techniques of leave-one-out cross validation (LOOCV) and Monte Carlo cross validation (repeated random test-validation splits, typically independent and in a fixed proportion). Given a dataset of size $m_{tr}$, LOOCV would require $m_{tr}$ regression fits which can be inefficient for large dataset size. Alternately, we can consider draws from a distribution $\cD_{LOO}$ which generates problem instances $P$ from a fixed dataset $(X,y)\in\R^{m+1\times p}\times\R^{m+1}$ by uniformly selecting $j\in[m+1]$ and setting $P=(X_{-j*},y_{-j},X_{j*},y_j)$. Theorem~\ref{thm:en-sc} now implies that $\Tilde{O}(p^2/\epsilon^2)$ iterations are enough to determine an ElasticNet parameter $\hat{\lambda}$ with loss within $\epsilon$ (with high probability) of the parameter $\lambda^*$ obtained from running the full LOOCV. Similarly, we can define a distribution $\cD_{MC}$ to capture the Monte Carlo cross validation procedure and determine the number of iterations sufficient to get an $\epsilon$-approximation of the loss corresponding parameter selection with arbitrarily large number of runs of the procedure.
Thus, in a very precise sense, our results answer the question of how much cross-validation is enough to effectively implement the above techniques.

\begin{remark}
While our result implies polynomial sample complexity, the question of learning the provably near-optimal parameter efficiently (even in output polynomial time) is left open. For the special cases of LASSO $(\lambda_2=0)$ and Ridge $(\lambda_1=0)$, the piece boundaries of the piecewise polynomial dual class (loss) function may be computed efficiently (using the LARS-LASSO algorithm of \cite{efron2004least} for LASSO, and solving linear systems and locating roots of polynomials for Ridge). This applies to online and classification settings in the following sections as well.
\end{remark}

\subsection{Online Learning}\label{sec:en-regression-online}

We will now extend our results to learning the regularization coefficients given an online sequence of regression problems, such as when one needs to solve a new regression problem each day. Unlike the distributional setting above, we will not assume any problem distribution and our results will hold for an adversarial sequence of problem instances. We will need very mild assumptions on the data, namely boundedness of feature and prediction values and  `smoothness' of predictions (formally stated as Assumptions \ref{ass:boundedness} and \ref{ass:smoothness}), while our distributional results above hold for worst-case problem datasets.


We will need two mild assumptions on the datasets in our problem instances for our results to hold. Our first assumption is that all feature values and predictions are bounded, for training as well as validation examples (Assumption \ref{ass:boundedness} above).
 We will need the following definition to state our second assumption. Roughly speaking the definition below captures smoothness of a distribution.

\begin{definition}\label{def:k-bounded}A continuous probability distribution is said to be {\it $\kappa$-bounded} if the probability density function $p(x)$ satisfies $p(x)\le \kappa$ for any $x$ in the sample space. 
\end{definition}
\noindent For example, the {\it normal} distribution $\cN(\mu,\sigma^2)$ with mean $\mu$ and standard deviation $\sigma$ is $\frac{1}{\sigma\sqrt{2\pi}}$-bounded. 
We assume that the predicted variable $y$ in the training set comes from a $\kappa$-bounded (i.e. smooth) distribution, {which does not require the strong tail decay of sub-Gaussian distributions \cite{zhang2009some,candes2009near}. Moreover, the online adversary is allowed to change the distribution as long as it is $\kappa$-bounded.} Note that our assumption also  captures common data preprocessing steps, for example the \textsf{jitter} parameter in the popular Python library scikit-learn \cite{scikit-learn} adds a uniform noise to the $y$ values to help model stability. The assumption is formally stated as follows:

\begin{assumption}[Smooth predictions]
The predicted variables $y^{(i)}$ in the training set are drawn from a joint $\kappa$-bounded distribution, i.e. for each $i$, the variables $y^{(i)}$ have a joint distribution with probability density bounded by $\kappa$.
\label{ass:smoothness}
\end{assumption}

\noindent Under these assumptions, we can show that it is possible to learn the ElasticNet parameters with sublinear expected regret when the problem instances arrive online. The  learning algorithm (Algorithm \ref{alg:ddreg}) that achieves this regret is a continuous variant of the classic Exponential Weights algorithm \cite{cesa2006prediction,balcan2018dispersion}. It samples points in the domain with probability inversely propotional to the exponentiated loss. To formally state our result, we will need the following definition of {\it dispersed} loss functions. Informally speaking, it captures how amenable a set of non-Lipschitz functions is to online learning by measuring the worst rate of occurrence of non-Lipschitzness (or discontinuities) between any pair of points in the domain. \cite{balcan2018dispersion,sharma2020learning,dick2020semi} show that dispersion is necessary and sufficient for learning piecewise Lipschitz functions.
\begin{definition}\label{def:dispersion} Dispersion \cite{dick2020semi}.
The sequence of random loss functions $l_1, \dots,l_T$ is $\beta$-{\it dispersed} for the Lipschitz constant $L$ if, for all $T$ and for all $\epsilon\ge T^{-\beta}$, we have that, in expectation, at most
$\Tilde{O}(\epsilon T)$ functions (the soft-O notation suppresses dependence on quantities beside $\epsilon,T$ and $\beta$, as well as logarithmic terms)
are not $L$-Lipschitz for any pair of points at distance $\epsilon$ in the domain $\cC$. That is, for all $T$ and for all $\epsilon\ge T^{-\beta}$,
$
    \bbE\Big[
\max_{\substack{\rho,\rho'\in\cC\\\norm{\rho-\rho'}_2\le\epsilon}}\big\lvert
\{ t\in[T] \mid l_t(\rho)-l_t(\rho')>L\norm{\rho-\rho'}_2\} \big\rvert \Big] 
\le  \Tilde{O}(\epsilon T)
$.
\end{definition}

\noindent Our key contribution is to show that the loss sequence is dispersed (Definition \ref{def:dispersion}) under the above assumptions. This involves establishing additional structure for the problem, specifically about the location of boundary functions in the piecewise structure from Theorem \ref{lem:en-structure}. This stronger characterization coupled with results from \cite{dick2020semi} on dispersion of algebraic discontinuities completes the proof. 

\begin{theorem}
 \label{thm:en-regression-dispersion} 
Suppose Assumptions \ref{ass:boundedness} and \ref{ass:smoothness} hold. Let $l_1,\dots, l_T:(0,\lambda_{\max} )^2\rightarrow\R_{\ge 0}$ denote an independent sequence of losses (e.g. fresh randomness is used to generate the validation set features in each round) as a function of the ElasticNet regularization parameter $\lambda=(\lambda_1,\lambda_2)$, $l_i(\lambda)=l_r(\hat{\beta}^{(X^{(i)},y^{(i)})}_{\lambda,f_{EN}},(\Xv^{(i)},\yv^{(i)}))$. The sequence of functions is $\frac{1}{2}$-dispersed, and there is an online algorithm with $\Tilde{O}(\sqrt{T})$\footnote{The $\Tilde{O}(\cdot)$ notation hides dependence on logarithmic terms, as well as on quantities other than $T$.} expected regret. The result also holds for loss functions adjusted by information criteria AIC and BIC.
\end{theorem}

\begin{proofoutline}
We start with the $(\cF,\cG,p3^p)$-piecewise decomposable structure for the dual class function $\Hen^*$ from Theorem \ref{lem:en-structure}. Observe that the rational piece functions in $\cF$ do not introduce any new discontinuities since the denominator polynomials do not have positive roots. For each of two types of boundary functions in $\cG$ (corresponding to leaving/entering the equicorrelation set) we show that the discontinuities between any pair of points $\lambda,\lambda'$ lie along the roots of polynomials with non-leading coefficients bounded and smoothly distributed (bounded joint density). This allows us to use results from \cite{dick2020semi} to establish dispersion, and therefore online learnability.
\end{proofoutline}

\begin{algorithm}[t]
\caption{{Data-driven Regularization ($\zeta$)}}
\label{alg:ddreg}
\begin{algorithmic}[1]
\STATE {\bfseries Input:} Problems $(X^{(i)},y^{(i)})$ and regularization penalty function $f$.
\STATE {\bfseries Hyperparameter:} step size parameter  $\zeta \in (0, 1]$.
\STATE {\bfseries Output:} Regularization parameter $(\lambda_i)_{i\in[T]}\in C$, $C\subset\R^+$ (LASSO/Ridge) or $C\subset{\R^+}^2$ (ElasticNet).
\STATE{Set $w_1(\lambda)=1$ for all $\lambda\in C$.}
\FOR{$i=1,2,\dots,T$}
\STATE{$W_i:=\int_{C}w_i(\lambda)d\lambda$.}
\STATE{Sample $\lambda$ with probability
        $p_{t}(\lambda)=\frac{w_i(\lambda)}{W_i}$, output as $\lambda_i$.}
\STATE{Compute average loss function $l_i(\lambda)=\frac{1}{|y^{(i)}|}l(\hat{\beta}_{\lambda,f},(X^{(i)},y^{(i)}))$.}
\STATE{For each $\lambda\in C, \text{ update weights }w_{i+1}(\lambda)=e^{\zeta (1-l_i(\lambda))}w_{i}(\lambda)$.}
\ENDFOR
\end{algorithmic}
\end{algorithm}

{We remark that the above result holds for arbitrary training features and validation sets in the problem
sequence that satisfy our assumptions, in particular the losses are only assumed to be independent but
not identically distributed. In contrast, the results in the previous section needed them to be drawn
from the same distribution.} Also the parameters need to be selected online, and cannot be changed for already seen instances. This setting captures interesting practical settings where the set of features (including feature dimensions) and the relevant training set (including training set size) may change over the online sequence. It is not clear how usual model selection techniques like cross-validation may be adapted to these challenging settings.

\section{Extension to Regularized Least Squares Classification}\label{sec:classification}
Regression techniques can also be used to train binary classifiers by using an appropriate threshold on top of the regression estimate. Intuitively, regression learns a linear mapping which projects the datapoints onto a one-dimensional space, i.e. a real number, after which a threshold may be applied to classify the points. The use of thresholds to make discrete classifications adds discontinuities to the empirical loss function. Thus, in general, the classification setting is more challenging as it already includes the piecewise structure in the regression loss. We provide statistical and online learning guarantees for Ridge and LASSO. For the ElasticNet we present the extensions needed to the arguments from the previous sections to obtain results in the classification setting.

More formally, we will restrict $y$ to $\{0,1\}^m$. The estimator $\hat{\beta}_{\lambda,f}$ is obtained as before, and the prediction on a test instance $x$ may be obtained by taking the sign of a thresholded regression estimate, $\sgn(\langle x, \hat{\beta}_{\lambda,f}\rangle-\tau)$, where $\sgn:\R\rightarrow\{0,1\}$ maps $x\in\R$ to $\bI\{x\ge 0\}$ and $\tau\in\R$ is the {\it threshold}. The threshold $\tau$  corresponds to the intercept or bias of the learned linear classifier, here we will treat it as a tunable hyperparameter (in addition to $\lambda_1,\lambda_2$)\footnote{We can still have a problem instance specific bias in $\beta$ using the standard trick of adding a unit feature to $X$, thus we generalize the common practice of using a fixed threshold. For example, the RidgeClassifier implementation in Python library scikit-learn 1.1.1 \cite{scikit-learn} assumes $y\in\{-1,+1\}^m$ and sets $\tau=0$.}. 
The average 0-1 loss over the dataset $(X,y)$ is given by $l_c(\hat{\beta}_{\lambda,f}, (X,y), \tau) = \frac{1}{m}\sum_{i=1}^m|y_i-\sgn(\langle X_i, \hat{\beta}_{\lambda,f}\rangle-\tau)|$\footnote{Squared loss and 0-1 loss are identical in this setting.}. Proofs from this section are in Appendix \ref{app:classification}.\looseness-1

\subsection{Distributional setting}\label{sec:classification:distrib}

The problem setting is the same as in Section \ref{sec:en-regression-distributional}, except that the labels $y$ are binary and we use threshold for prediction. We bound the pseudo-dimension for classification loss on these problem instances, which as before (c.f. Theorems~\ref{thm:pdim-en} and~\ref{thm:en-sc}) imply that polynomially many problem samples are sufficient to generalize well over the problem distribution $\cD$. For Ridge and LASSO we upper bound the number of discontinuities of the piecewise constant classification loss by determining the values of $\lambda$ where any prediction changes.

\begin{theorem}\label{thm:ridge-pdim}
 Let $\Hr^c$, $\Hl^c$ and $\Hen^c$ denote the set of loss functions for classification problems with at most $m$ examples and $p$ features, for linear classifiers regularized using Ridge, LASSO and ElasticNet regression respectively.
 \begin{itemize}[leftmargin=0.8cm,nosep]
    \item[(i)] $\pdim(\Hr^c)=O(\log mp)$
    \item[(ii)]  $\pdim(\Hl^c)=O(p\log m)$. Further, in the overparameterized regime ($p\gg m$), we have that $\pdim(\Hl^c) = O(m\log \frac{p}{m})$.
    \item[(iii)] $\pdim(\Hen^c)=O(p^2+p\log m)$.
\end{itemize}
\end{theorem}

\noindent The key difference with the bound for the regression loss in Theorem \ref{thm:pdim-en} is the additional $O(p\log m)$ term which corresponds to discontinuities induced by the thresholding in the regression based classifiers. We can establish a structure similar to Theorem \ref{lem:en-structure} in this case (Lemma \ref{lem:en-structure-c}).

\subsection{Online setting}\label{sec:online}

As in Section \ref{sec:en-regression-online}, we can define an online learning setting for classification. 
Note that the smoothness of the predicted variable is not meaningful here, since $y$ is a binary vector. Instead we will assume that the validation examples have smooth feature values. Intuitively this means that small perturbations to the feature values does not meaningfully change the problem.

\begin{assumption}[Smooth validation features]
The feature values $(\Xv^{(i)})_{jk}$ in the validation examples are drawn from a joint $\kappa$-bounded distribution.
\label{ass:smooth-features}
\end{assumption}

\noindent Under the assumption, we show that we can learn the regularization parameters online, for each of Ridge, LASSO and ElasticNet estimators. The proofs are straightforward extensions of the structural results developed in the previous sections, with minor technical changes to use the above validation set feature smoothness instead of Assumption \ref{ass:smoothness}, and are deferred to the appendix.

\begin{theorem}\label{thm:ridge-dispersion}
Suppose Assumptions \ref{ass:boundedness} and \ref{ass:smooth-features} hold. Let $l_1,\dots, l_T:(0,H]^{d}\times[-H,H]\rightarrow\R$ denote an independent sequence of losses as a function of the regularization parameter $\lambda$, $l_i(\lambda,\tau)=l_c(\hat{\beta}_{\lambda,f},(X^{(i)},y^{(i)}),\tau)$. 
If $f$ is given by $f_1$ (LASSO), $f_2$ (Ridge), or $f_{EN}$ (ElasticNet) then
the sequence of functions is $\frac{1}{2}$-dispersed and there is an online algorithm with $\Tilde{O}(\sqrt{T})$ expected regret.
\end{theorem}

\section{Conclusions and Future Work}\label{sec:conclusion}

We obtain a novel structural result for the ElasticNet loss as a function of the tuning parameters. Our characterization is useful in giving upper bounds for the sample complexity of learning the parameters from multiple regression problem instances \blue{(i.e.\ different datasets, possibly corresponding to different tasks)} from the same problem domain. Efficient algorithms are immediate from our results for Ridge and LASSO. For the ElasticNet we show generalization and online regret guarantees, but efficient implementation of the algorithms is an interesting question for further work. Also we show general learning-theoretic guarantees, i.e. without any significant restrictions on the data-generating distribution, in learning from multiple problems. The problems may be drawn i.i.d. from an arbitrary {\it problem} distribution, or even arrive in an online sequence but with some smoothness properties. It is unclear if such general guarantees may be given for tuning parameters for the more standard setting of tuning over a single training set generated by i.i.d. draws from an {\it example} distribution, or how such  guarantees can be combined with our results.\looseness-1

\section*{Acknowledgments}

This material is based on work supported by the National Science Foundation under grants CCF-1910321, IIS-1705121, IIS-1838017, IIS-1901403, IIS-2046613, IIS-2112471, and SES-1919453; the Defense Advanced Research Projects Agency under cooperative agreement HR00112020003; a Simons Investigator Award; an AWS Machine Learning Research Award; an Amazon Research Award; a Bloomberg Research Grant;  a Microsoft Research Faculty Fellowship; funding from Meta, Morgan Stanley, and Amazon; and a Facebook PhD Fellowship.
Any opinions, findings and conclusions or recommendations expressed in this material are those of the author(s) and do not necessarily reflect the views of any of these funding agencies.

\bibliographystyle{alpha}
\bibliography{main}

\newpage
\section*{Appendix}
\appendix

\section{A classic Generalization Bound}\label{app:learning-theory}

The pseudo-dimension (also known as the {\it Pollard dimension}) is a generalization of the VC-dimension to real-valued functions, and may be defined as follows.

\begin{definition}[Pseudo-dimension \cite{pollard2012convergence}]\label{def:pd} Let $\cH$ be a set of real valued functions from input space $\cX$. We say that
$C = (x_1, \dots, x_n)\in \cX^n$ is pseudo-shattered by $\cH$ if there exists a vector
$r = (r_1, \dots, r_n)\in\R^n$ (called ``witness”) such that for all
$b= (b_1, \dots, b_n)\in \{\pm 1\}^n $ there exists $h_b\in \cH$ such that $\text{sign}(h_b(x_i)-r_i)=b_i$. Pseudo-dimension of $\cH$, denoted by $\pdim(\cH)$, is the cardinality of the largest set
pseudo-shattered by $\cH$.
\end{definition}

\noindent The following theorem connects the sample complexity of uniform learning over a class of real-valued functions to the pseudo-dimension of the class. Let $h^*:\cX\rightarrow\{0,1\}$ denote the target concept. We say $\cH$ is $(\epsilon,\delta)$-uniformly learnable\footnote{$(\epsilon,\delta)$-uniform learnability with $n$ samples implies $(\epsilon,\delta)$-PAC learnability with $n$ samples.} with sample complexity $n$ if, for every distribution $\cD$, given a sample $S$ of size $n$, with probability $1 - \delta$, $\big\lvert \frac{1}{n}\sum_{s\in S}|h(s)-h^*(s)| - \bbE_{s\sim\cD}[|h(s)-h^*(s)|] \big\rvert < \epsilon$ for every $h\in\cH$.

\begin{theorem}[\cite{anthony1999neural}] \label{thm:pdim-generalization}
 Suppose $\cH$ is a class of real-valued functions with range in $[0,H]$ and pseudo-dimension $\pdim(\cH)$. For every $\epsilon>0,\delta\in (0,1)$, the sample complexity of $(\epsilon,\delta)$-uniformly learning the class $\cH$ is $O\left(\left(\frac{H}{\epsilon}\right)^2\left(\pdim(\cH)\ln\frac{H}{\epsilon}+\ln\frac{1}{\delta}\right)\right)$.
\end{theorem}

\section{Known characterization of LASSO solutions}
We will review some properties of LASSO solutions from prior work that are useful in proving our results. Let $(X,y)$ with $X=[\x_1,\dots,\x_p]\in\R^{m\times p}$ and $y\in\R^{m}$ denote a (training) dataset consisting of $m$ labeled examples with $p$ features. As noted in Section \ref{sec:prelim}, LASSO regularization may be formulated as the following optimization problem.

\[\min_{\beta\in\R^p}\norm{y-X\beta}_2^2+\lambda_1||\beta||_1,\]
\noindent where $\lambda_1\in\R^+$ is the regularization parameter. Dealing with the case $\lambda_1=0$ (i.e. Ordinary Least Squares) is not difficult, but is omitted here to keep the statements of the definitions and results simple. We will use the following well-known facts about the solution of the LASSO optimization problem \cite{fuchs2005recovery,tibshirani2013lasso}. Applying the Karush-Kuhn-Tucker (KKT) optimality conditions to the problem gives,

\begin{lemma}[KKT Optimality Conditions for LASSO] $\beta^*\in \argmin_{\beta\in\R^p}\norm{y-X\beta}_2^2+\lambda_1||\beta||_1$ iff for all $j\in[p]$,
\begin{align*}
    \x_j^T(y-X\beta^*)&=\lambda_1\sign(\beta^*), \text{ if }\beta^*_j\ne 0,\\
    |\x_j^T(y-X\beta^*)|&\le\lambda_1, \text{ otherwise. }
\end{align*}
\label{lem:lasso-kkt}
\end{lemma}
\noindent Here $\x_j^T(y-X\beta^*)$ is simply the correlation of the the $j$-th covariate with the residual $y-X\beta^*$  (when $y,X$ have been standardized). This motivates the definition of {\it equicorrelation sets} of covariates (Definition \ref{def:ec}).

\noindent In terms of the equicorrelation set and the equicorrelation sign vector, the characterization of the LASSO solution in Lemma \ref{lem:lasso-kkt} implies
\[X_{\cE}^T(y-X_{\cE}\beta^*_{\cE})=\lambda_1 s.\]

\noindent This implies a necessary and sufficient condition for the uniqueness of the LASSO solution, namely that $X_{\cE}$ is full rank for all equicorrelation sets $\cE$ \cite{tibshirani2013lasso}. Our results will hold if the dataset $X$ satisfies this condition, but for simplicity we will use the a simpler (and possibly more natural) sufficient condition involving the {\it general position}.

\begin{definition} A matrix $X\in\R^{m\times p}$ is said to have its columns in the general position if the affine span of any $k\le m$ points $(\sigma_i\x_{j_i})_{i\in[k],\{j_i\}_i=J\subseteq[p]}$ for arbitrary signs $\sigma_{[k]}\in\{-1,1\}^{k}$ and subset $J$ of the columns of size $k$, does not contain any element of $\{\x_i\mid i\notin J\}$.
\label{def:gp}
\end{definition}

\noindent Finally, we state the following useful characterization of the LASSO solutions in terms of the equicorrelation sets and sign vectors.

\begin{lemma}[\cite{tibshirani2013lasso}, Lemma 3] If the columns of $X$ are in general position, then for any $y$ and $\lambda_1>0$, the LASSO solution
is unique and is given by
\[\beta^*_{\cE}=(X_\cE^TX_\cE)^{-1}(X_{\cE}^Ty-\lambda_1 s), \beta^*_{[p]\setminus\cE}=0.\]
\label{lem:lasso}
\end{lemma}

\noindent We remark that Lemma \ref{lem:lasso} does not give a way to compute $\beta^*$ for a given value of $\lambda_1$, since $\cE$ and $s$ depend on $\beta^*$, but still gives a property of $\beta^*$ that is convenient to use. In particular, since we have at most $3^p$ possible choices for $(\cE,s)$, this implies that the LASSO solution $\beta^*(\lambda_1)$ is a piecewise linear function of $\lambda_1$, with at most $3^p$ pieces (for $\lambda_1>0$). Following popular terminology, we will refer to this function as a {\it solution path} of LASSO for the given dataset $(X,y)$. LARS-LASSO of \cite{zou2005regularization} is an efficient algorithm for computing the {\it solution path} of LASSO.

\begin{corollary}Let $X$ be a matrix with columns in the general position. If the unique LASSO solution for the dataset $(X,y)$ is given by the function $\beta^*:\R^+\rightarrow \R^p$, then $\beta^*$ is piecewise linear with at most $3^p$ pieces given by Lemma \ref{lem:lasso}.

\label{cor:lasso}
\end{corollary}

\section{Lemmas and proof details for Section \ref{sec:regression}}\label{app:regression}

We start with a helper lemma that characterizes the solution of the ElasticNet in terms of equicorrelation sets and sign vectors.

\begin{lemma}\label{lem:en-equicorrelation}
Let $X$ be a matrix with columns in the general position, and $\lambda=(\lambda_1,\lambda_2)\in(0,\infty)\times (0,\infty)$. Then the ElasticNet solution $\hat{\beta}_{\lambda,f_{EN}}\in \argmin_{\beta\in\R^p}\norm{y-X\beta}_2^2+\langle\lambda ,f_{\text{EN}}(\beta)\rangle$ is unique for any dataset $(X,y)$ and satisfies
\[\hat{\beta}_{\lambda,f_{EN}} = (X_{\cE}^TX_{\cE}+\lambda_2I_{|\cE|})^{-1}X_{\cE}^Ty-\lambda_1(X_{\cE}^TX_{\cE}+\lambda_2I_{|\cE|})^{-1}s\]
for some $\cE\subseteq[p]$ and $s\in\{-1,1\}^{|\cE|}$.
\end{lemma}
\begin{proof}
We start with the well-known characterization of the ElasticNet solution as the solution of a LASSO problem on a transformed dataset, obtained using simple algebra \cite{zou2005regularization}. Given any dataset $(X,y)$, the ElasticNet coefficients $\hat{\beta}_{\lambda,f_{EN}}$ are given by $\hat{\beta}_{\lambda,f_{EN}}=\frac{1}{\sqrt{1+\lambda_2}}\hat{\beta}^*_\lambda$\footnote{This corresponds to the ``naive" ElasticNet solution in the terminology of \cite{zou2005regularization}. They also define an ElasticNet `estimate' given by $\sqrt{1+\lambda_2}\hat{\beta}^*_\lambda$ with nice properties, to which our analysis is easily adapted.} where $\hat{\beta}^*_\lambda$ is the solution for a LASSO problem on a modified dataset $(X^*,y^*)$

\[\hat{\beta}^*_\lambda=\argmin_{\beta}\norm{y^*-X^*\beta}_2^2+\lambda_1^* f_1(\beta),\]

\noindent with $X^*=\frac{1}{\sqrt{1+\lambda_2}}\begin{pmatrix}X\\ \sqrt{\lambda_2}I_p\end{pmatrix}$, $y^*=\begin{pmatrix}y\\ 0\end{pmatrix}$, and $\lambda_1^*=\frac{\lambda_1}{\sqrt{1+\lambda_2}}$.

If the columns of $X$ are in general position (Definition \ref{def:gp}), then the same is true of $X^*$. 
For $\cE\subseteq[p]$, note that ${X^*_\cE}^TX^*_\cE=\frac{1}{1+\lambda_2}(X_{\cE}^TX_{\cE}+\lambda_2I_{|\cE|})$ and ${X^*_\cE}^Ty^*=\frac{1}{\sqrt{1+\lambda_2}}X_{\cE}^Ty$. By Lemma \ref{lem:lasso}, if $\cE$ denotes the equicorrelation set of covariates and $s\in\{-1,1\}^{|\cE|}$ the equicorrelation sign vector for the LASSO problem, then the ElasticNet solution  is  given by
\[\hat{\beta}_{\lambda,f_{EN}} = c_1-c_2{\lambda_1},\]

\noindent where $$c_1=\frac{1}{\sqrt{1+\lambda_2}} ({X^*_{\cE}}^TX^*_{\cE})^{-1}{X^*_{\cE}}^Ty^*=(X_{\cE}^TX_{\cE}+\lambda_2I_{|\cE|})^{-1}X_{\cE}^Ty,$$ and $$c_2= \frac{1}{1+\lambda_2}({X^*_{\cE}}^TX^*_{\cE})^{-1}s=(X_{\cE}^TX_{\cE}+\lambda_2I_{|\cE|})^{-1}s.$$ 
\end{proof}

\noindent The following lemma helps determine the dependence of ElasticNet solutions on $\lambda_2$.

\noindent{\bf Lemma \ref{lem:gram-inv} (restated).}
Let $A$ be an $r\times s$ matrix. Consider the matrix $B(\lambda)=(A^TA+\lambda I_s)^{-1}$ for $\lambda>0$. \begin{itemize}
    \item[1.] Each entry of $B(\lambda)$ is a rational polynomial $P_{ij}(\lambda)/Q(\lambda)$ for $i,j\in[s]$ with each $P_{ij}$ of degree at most $s-1$, and $Q$ of degree $s$.
    \item[2.] Further, for $i=j$, $P_{ij}$ has degree $s-1$ and leading coefficient 1, and for $i\ne j$ $P_{ij}$ has degree at most $s-2$. Also, $Q(\lambda)$ has leading coefficient $1$. 
\end{itemize}

\begin{proof}
Let $G=A^TA$ be the Gramian matrix. $G$ is symmetric and therefore diagonalizable, and the diagonalization gives the eigendecomposition $G=E\Lambda E^{-1}$. Thus we have
\[(A^TA+\lambda I_s)^{-1} = (E\Lambda E^{-1}+\lambda EE^{-1})^{-1}=E(\Lambda+\lambda I_s)^{-1}E^{-1}\]
But $\Lambda$ is the diagonal matrix $\diag(\Lambda_{11},\dots,\Lambda_{ss})$, and therefore $(\Lambda+\lambda I_s)^{-1}=\diag((\Lambda_{11}+\lambda)^{-1},\dots,(\Lambda_{ss}+\lambda)^{-1})$. This implies the desired characterization, with $Q(\lambda)=\Pi_{i\in[s]}(\Lambda_{ii}+\lambda)$ and
$$P_{ij}(\lambda)=Q(\lambda)\sum_{k=1}^s\frac{E_{ik}(E^{-1})_{kj}}{\Lambda_{kk}+\lambda}=\sum_{k=1}^s\left(E_{ik}(E^{-1})_{kj}\Pi_{i\in[s]\setminus k}(\Lambda_{ii}+\lambda)\right),$$
with coefficient of $\lambda^{s-1}$ in $P_{ij}(\lambda)$ equal to $\sum_{k=1}^sE_{ik}(E^{-1})_{kj}=\mathbb{I}\{i=j\}$.
\end{proof}

\subsection{Tuning the ElasticNet -- Distributional setting}\label{app:en-regression-distributional}
{We first present some terminology from algebraic geometry which will be useful in our proofs.

\begin{definition}[Semialgebraic sets, Algebraic curves.] A semialgebraic subset of $\R^n$ is a finite union of sets of the form $\{x\in\R^n\mid p_i(x)\ge 0 \text{ for each }i\in [m]\}$, where $p_1,\dots,p_m$ are polynomials. An algebraic curve is the zero set of a polynomial in two dimensions.
\label{def:semi-algebriac}
\end{definition}

}
\noindent The result of Theorem \ref{lem:en-structure} motivates the following results for bounding the complexity of dual piece functions and dual boundary functions, which can be used to bound the pseudo-dimension of $\Hen$ (Theorem \ref{thm:pdim-en}) using the following remarkable result from \cite{balcan2021much}.

\begin{theorem}[\cite{balcan2021much}]
 If the dual function class $\cH^*$ is $(\cF,\cG,k)$-piecewise decomposable, then the pseudo-dimension of $\cH$ may be bounded as
 \[\pdim(\cH)=O((\pdim(\cF^*)+d_{\cG^*})\log(\pdim(\cF^*)+d_{\cG^*})+d_{\cG^*}\log k),\]
 where $d_{\cG^*}$ denotes the VC dimension of dual class boundary function $\cG^*$.
 \label{thm:pdim-dual}
\end{theorem}

\noindent We will first prove a useful lemma that bounds the number of pieces into which a finite set of algebraic curves with bounded degrees may partition $\R^2$.

\begin{lemma}\label{lem:ac-cells}
Let $\cH$  be a collection of $k$ functions $h_i:\R^2\rightarrow\R$ that map $(x,y)\mapsto q_i(x,y)$ where $q_i$ is a bivariate polynomial of degree at most $d$, for $i\in[k]$, then $\R^2\setminus\{(x,y)\mid q_i(x,y)=0 \text{ for some } i\in[k]\}$ may be partitioned into at most $(kd+1)\left(d^2{k\choose 2}+2kd(d-1)+1\right)=O(d^3k^3)$ disjoint sets such that the sign pattern $(\mathbb{I}\{q_i(x,y)>0\})_{i\in[k]}$ is fixed over any set in the partition.
\end{lemma}
\begin{proof}
Assume WLOG that the curves are in the general position. Simple applications of Bezout's theorem (which states that, in general, two algebraic curves of degrees $m$ and $n$ intersect in at most $mn$ points) imply that there are at most $d^2{k\choose 2}$ points where any pair of curves from the set $\{q_i(x,y)\}_{i\in[k]}$ may intersect, and at most $2kd(d-1)$ points of extrema (i.e. points $p_0=(x_0,y_0)$ on the curve $f$ such that there is an open neighborhood $N$ around $p_0$ in which $x_0\in \argmin_{(x,y)\in N\cap f}x$, or $x_0\in \argmax_{(x,y)\in N\cap f}x$, or $y_0\in \argmin_{(x,y)\in N\cap f}y$, or $y_0\in \argmax_{(x,y)\in N\cap f}y$) for the $k$ algebraic curves. Let $\cP$ denote the set of these $\le d^2{k\choose 2}+2kd(d-1)$ points.

Now a horizontal line $y=c$ will have the exact same set of intersections with all the curves in $\cH$ as a line $y=c'$, and in the same order (including multiplicities), if none of the points in $\cP$ lie between these lines. There are thus at most $|\cP|+1$ distinct sequences of the $k$ curves that may correspond to the intersection sequence of any horizontal line. Moreover, any such horizontal line may intersect any curve in the set at most $d$ times (since a polynomial in degree $d$ has at most $d$ zeros), or at most $kd$ intersections with all the curves. Summing up over the distinct intersection sequences, we have at most $(kd+1)(|\cP|+1)$ distinct sign patterns induced by the set of curves.
\end{proof}

\noindent We will now use Lemma \ref{lem:ac-cells} to bound the pseudo-dimension of the relevant function classes (Theorem \ref{lem:en-structure}).

\begin{lemma}\label{lem:f-pdim}
Let $\cF^* = \{f^*_{q_1,q_2} : \R^2 \rightarrow \R\}$ be a function class consisting of rational polynomial functions $f^*_{q_1,q_2} : (\lambda_1,\lambda_2)\mapsto \frac{q_1(\lambda_1,\lambda_2)}{q_2(\lambda_1,\lambda_2)}$, where $q_1,q_2$ have degrees at most $d$. Then $\pdim(\cF^*)=O(\log d)$.
\end{lemma}
\begin{proof}
Suppose that $\pdim(\cF^*)=N$. Then there exist functions 
$f^*_1,\dots,f^*_N\in\cF^*$ and real-valued witnesses $(r_1,\dots,r_N)\in\R^N$ such that for every subset $T\subseteq[N]$, there exists a parameter setting $\lambda_T=(\lambda_1,\lambda_2)\in\R^2$ such that $f^*_i(\lambda_T)\ge r_i$ if and only if $i\in T$. In other words, we have a set of $2^N$ parameters (indexed by $T$) that induce all possible labelings of the binary vector $(\mathbb{I}\{f^*_i(\lambda_T)\ge r_i\})_{i\in[N]}$.

But $f^*_i(\lambda)\ge r_i$ are semi-algebraic sets bounded by $N$ algebraic curves of degree at most $d$. By Lemma \ref{lem:ac-cells}, there are at most $O(d^3N^3)$ different sign-patterns induced by $N$ algebraic curves over all possible values of $\lambda\in\R^2$. In particular, the number of distinct sign patterns over $\lambda\in\{\lambda_T\}_{T\subseteq [N]}$ is also $O(d^3N^3)$. Thus, we conclude $2^N=O(d^3N^3)$, or $N=O(\log d)$.
\end{proof}

\begin{lemma}\label{lem:g-vcdim}
Let $R[x_1,x_2,\dots,x_d]_{D}$ denote the set of all real polynomials in $d$ variables of degree at most $D$ in $x_1$, and degree at most 1 in  $x_2,\dots,x_d$. Further, let $P_{d,D} = \{\{x\in R^d :p(x)\ge 0\}\mid p\in R[x_1,x_2,\dots,x_d]_{D}\}$. The VC-dimension of the set system $(\R^d,P_{d,D})$ is $O(dD)$.
\end{lemma}

\begin{proof}
   We will employ a standard {\it linearization} argument \cite{cover1965geometrical} that reduces the problem to bounding the VC dimension of halfspaces in higher dimensions. Let $M$ be the set of all possible non-constant monomials of degree at most $D$ in $x_1$, and at most one in $x_2,\dots,x_d$. For example, when $d=3$ and $D=2$, we have $M=\{x_1,x_2,x_3,x_1x_2,x_1x_3,x_1^2,x_1^2x_2,x_1^2x_3\}$. Note that $|M|=(D+1)d-1$. Indeed for each $x_1^i$ for $i=0,\dots,D$ we obtain a monomial by multiplying with each of $\{1,x_2,\dots,x_d\}$. Excluding the constant monomial gives the result. The linearization we use is a map $\phi:\R^d\rightarrow\R^{|M|}$ which indexes the coordinates by monomials in $M$. For example when $d=3$ and $D=2$, $\phi(x_1,x_2,x_3)=(x_1,x_2,x_3,x_1x_2,x_1x_3,x_1^2,x_1^2x_2,x_1^2x_3)$.
   
   Now, if $S \in \R^d$ is shattered by $P_{d,D}$, then $\phi(S)$ is shattered by half-spaces in $\R^{|M|}$. Indeed, suppose $p=p_0+\langle \mathbf{p}, \phi(x_1,\dots,x_d)\rangle \in P_{d,D}$ (for $\mathbf{p}\in\R^{|M|}$) is a polynomial that is positive over some $T\subseteq S$ and negative over $S\setminus T$. Define halfspace $h_p\in\R^{|M|}$ as $\{y\in\R^{|M|}\mid p_0+\langle \mathbf{p},y\rangle\ge 0\}$. Clearly $h_p\cap \phi(S)=\phi(T)$, and in general $S$ is shattered by halfspaces in $\R^{|M|}$. Using the well-known result for the VC-dimension of halfspaces we have that the VC-dimension of $P_{d,D}$ over $\R^d$ is $(D+1)d$. 
\end{proof}


\noindent{\bf Theorem \ref{thm:pdim-en} (restated).} {\it
$\pdim(\Hen)=O(p^2)$. Further, $\pdim(\Hen^{\textsf{AIC}})=O(p^2)$ and $\pdim(\Hen^{\textsf{BIC}})=O(p^2)$.}
\begin{proof}
By Theorem \ref{lem:en-structure}, the dual class $\Hen^*$ of $\Hen$ is $(\cF,\cG,p3^p)$-piecewise decomposable, with $\cF = \{f_{q_1,q_2} : \cL \rightarrow \R\}$ consisting of rational polynomial functions $f_{q_1,q_2} : l_\lambda\mapsto \frac{q_1(\lambda_1,\lambda_2)}{q_2(\lambda_2)}$, where $q_1,q_2$ have degrees at most $2p$, and $\cG=\{g_r:\cL\rightarrow \{0,1\}\}$ consisting of semi-algebraic sets bounded by algebraic curves $g_r : u_\lambda \mapsto \mathbb{I}\{r(\lambda_1,\lambda_2)<0\}$, where $r$ is a polynomial of degree 1 in $\lambda_1$ and at most $p$ in $\lambda_2$.

Now by Lemma \ref{lem:f-pdim}, we have $\pdim(\cF^*)=O(\log p)$, and by Lemma \ref{lem:g-vcdim} the VC dimension of the dual boundary class is $d_{\cG^*}=O(p)$. A straightforward application of Theorem \ref{thm:pdim-dual} yields \[\pdim(\cH)=O(p\log p+p\log(p3^p))=O(p^2).\]

\noindent The dual classes ${(\Hen^{\textsf{AIC}})}^*$ and ${(\Hen^{\textsf{BIC}})}^*$ also follow the same piecewise decomposable structure given by Theorem \ref{lem:en-structure}. This is because in each piece the equicorrelation set $\cE$, and therefore $||\beta||_0=|\cE|$ (by Lemma \ref{lem:lasso}) is fixed. Thus we can keep the same boundary functions, and the function value in each piece only changes by a constant (in $\lambda$) and is therefore also a rational function with the same degrees. The above argument then implies an identical upper bound on the pseudo-dimensions.
\end{proof}

\blue{
\noindent The following lemma  shows that under mild boundedness assumptions on the data and the search space of hyperparameters, the validation loss function class $\Hen$ is uniformly bounded by some constant $H > 0$.
\begin{lemma} \label{lemma:bounded-EN}
    Under Assumption \ref{ass:boundedness}, there exists a  constant $H > 0$ so that for all $h(\lambda, \cdot) \in \Hen = \{h(\lambda, \cdot): \Pi_{m, p} \rightarrow \R_{\geq 0} \mid \lambda \in [\lambda_{\min}, \lambda_{\max}]\}$, we have $\|h(\lambda, \cdot)\|_\infty = \sup_{\substack{P \in \Pi_{m, p}}}{h(\lambda, P)} \leq H$.
\end{lemma}
\begin{proof}
For any problem instance $P = (X, y, \Xv, \yv) \in \Pi_{m, p}$, and for any $\lambda = (\lambda_1, \lambda_2) \in [\lambda_{\min}, \lambda_{\max}]^2$, consider the optimization problem for training set $(X, y)$
\begin{equation} \label{eq:training-optimization-problem}
    \argmin_{\beta} F(\beta),
\end{equation}
where $F(\beta) = \frac{1}{2m}||{y - X\beta}||_2^2 + \lambda_1||\beta||_1 + \lambda_2||\beta||_2^2$. If we set $\beta = \Vec{0}$, we have
\begin{equation*}
    F(\Vec{0}) = \frac{1}{2m}||{y}||_2^2 \leq \frac{1}{2}R^2,
\end{equation*}
for some absolute constant $R$, due to Assumption \ref{ass:boundedness}. Let $\hat{\beta}_{(X, y)}(\lambda)$ be the optimal solution of \ref{eq:training-optimization-problem}, we have
\begin{equation*}
    \frac{1}{2}R^2 \geq F(\hat{\beta}_{(X, y)}(\lambda)) \geq \lambda_1||{\hat{\beta}_{(X, y)}(\lambda)}||_1 + \lambda_2||{\hat{\beta}_{(X, y)}(\lambda)}||_2^2.
\end{equation*}
Therefore, for any problem instance $P$, the solution of the training optimization problem $\hat{\beta}_{(X, y)}(\lambda)$ has bounded norm, i.e. $||{\hat{\beta}_{(X, y)}(\lambda)}||_1, ||{\hat{\beta}_{(X, y)}(\lambda)}||_2^2 \leq \frac{R^2}{2\lambda_{\min}}$, which implies
\begin{equation*}
    h(\lambda, P) = \frac{1}{2m}||{\yv - \hat{\beta}_{(X, y)}(\lambda)\Xv}||_2^2 \leq \frac{1}{2m}||{\yv}||_2^2 + \frac{1}{2m}||{\hat{\beta}_{(X, y)}(\lambda)}{\Xv}||_2^2 \leq H,
\end{equation*}
for some  constant $H$ (depends only on $R$ and $\lambda_{\min}$).     
\end{proof} 
}

\subsection{Tuning the ElasticNet -- Online learning}\label{app:en-regression-online}

At a high level, the plan is to show dispersion (Definition \ref{def:dispersion}) using the general recipe developed in \cite{dick2020semi}. The recipe may be summarized at a high level as follows.

\begin{itemize}
    \item[S1.]  Bound the probability density of the random set of
discontinuities of the loss functions. Intuitively this corresponds to computing the average number of loss functions that may be discontinuous along a path connecting any two points within distance $\epsilon$ in the domain.
\item[S2.] Use a VC-dimension based uniform convergence argument to transform this into a bound on the
dispersion of the loss functions.
\end{itemize}

\noindent Formally, we have the following theorems from \cite{dick2020semi}, which show how to use this technique when the discontinuities are roots of a random polynomial with bounded coefficients. The theorems implement steps S1 and S2 of the above recipe respectively.

\begin{theorem}[\cite{dick2020semi}]\label{thm:poly-roots}
Consider a random degree $d$ polynomial
$\phi$ with leading coefficient 1 and subsequent coefficients
which are real of absolute value at most $R$, whose joint
density is at most $\kappa$. There is an absolute constant $K_0$
depending only on $d$ and $R$ such that every interval $I$ of
length $\le\epsilon$ satisfies Pr($\phi$ has a root in $I$) $\le \kappa\epsilon/K_0$.
\end{theorem}

\begin{theorem}[\cite{dick2020semi}]\label{thm:VC-bound}
 Let $l_1, \dots, l_T : \R \rightarrow \R$ be independent piecewise $L$-Lipschitz functions, each having at most $K$ discontinuities. Let $D(T, \epsilon, \rho) = |\{1 \le t \le T \mid l_t\text{ is not }L\text{-Lipschitz on }[\rho - \epsilon, \rho + \epsilon]\}|$
be the number of functions that are not
$L$-Lipschitz on the ball $[\rho - \epsilon, \rho + \epsilon]$. Then we
have $E[\max_{\rho\in\R} D(T, \epsilon, \rho)] \le \max_{\rho\in\R} E[D(T, \epsilon, \rho)] +
O(\sqrt{T \log(TK)})$.
\end{theorem}

\noindent The following lemma provides useful extension to Lemma \ref{lem:gram-inv} for our online learning results.

\begin{lemma}\label{lem:gram-inv-coeffs}
Let $A$ be an $r\times s$ matrix with $R$-bounded max-norm, i.e. $||A||_{\infty,\infty}=\max_{i,j}|A_{ij}|\le R$. Then each entry of the matrix $(A^TA+\lambda I_s)^{-1}$ is a rational polynomial $P_{ij}(\lambda)/Q(\lambda)$ for $i,j\in[s]$ with each $P_{ij}$ of degree at most $s-1$, $Q$ of degree $s$, and all the coefficients have absolute value at most $r^s(Rs)^{2s}$.
\end{lemma}

\begin{proof}
Let $G=A^TA$ be the Gram matrix.
$$|G_{ij}|=|\sum_k{A_{ki}A_{kj}}|\le \sum_k{|A_{ki}A_{kj}|} \le rR^2,$$ by the triangle inequality and using $\max_{i,j}|A_{ij}|\le R$.  The determinant $\textsc{det}(A^TA+\lambda I_s)$ is a sum of $s!\le s^s$ signed terms, each a product of $s$ elements of the form $G_{ij}$ or $G_{ii}+\lambda$. Thus, in each of the $s!$ terms, the coefficient of $\lambda^k$ is a sum of at most ${s\choose s-k}\le s^k\le s^s$ expressions of the form $\Pi_{(i,j)\in S}G_{ij}$ with $|S|\le s-k$. Now $|\Pi_{(i,j)\in S}G_{ij}|\le (rR^2)^{|S|}\le (rR^2)^s$, and by triangle inequality the coefficient of $\lambda^k$ is upper bounded by $(rR^2)^s\cdot s^k\cdot s^s$ for any $k$. This establishes the bound on the coefficients of $Q(\lambda)$. A similar argument implies the upper bound for each $P_{ij}(\lambda)$.
\end{proof}

\noindent We will also need the following result, which is a simple extension of Lemma 24 from \cite{balcan2021data}.

\begin{lemma}\label{lem:bounded}
Suppose $X$ and $Y$ are real-valued random variables taking values in $[m, m + M]$ for some $m,M\in \R^+$ and suppose that their joint distribution is $\kappa$-bounded. Let $c$ be an absolute constant. Then,
\begin{itemize}
    \item[(i)] $Z=X+Y$ is drawn from a $K_1\kappa$-bounded distribution, where $K_1\le M$.
    \item[(ii)] $Z=XY$ is drawn from a $K_2\kappa$-bounded distribution, where $K_2\le M/m$.
    \item[(iii)]$Z=X-Y$ is drawn from a $K_1\kappa$-bounded distribution, where $K_1\le M$.
    \item[(iv)]$Z=X+c$ has a $\kappa$-bounded distribution, and $Z=cX$ has a $\frac{\kappa}{|c|}$-bounded distribution.
\end{itemize}
\end{lemma}
\begin{proof} Let $f_{X,Y}(x,y)$ denote the joint density of $X,Y$. (i) and (ii) are immediate from Lemma 24 from \cite{balcan2021data}, (iii) is a simple extension. Indeed, the cumulative density function for $Z$ is given by
    \begin{align*}
        F_Z(z)&=\bP(Z\le z)=\bP(X-Y\le z)=\bP(X\le z+Y)\\
        &=\int_{m}^{m+M}\int_{m}^{z+y}f_{X,Y}(x,y)dxdy.
    \end{align*}
    The density function for $Z$ can be obtained using Leibniz's rule as
    \begin{align*}
        f_Z(z)=\frac{d}{dz}F_Z(z)&=\frac{d}{dz}\int_{m}^{m+M}\int_{m}^{z+y}f_{X,Y}(x,y)dxdy\\&=\int_{m}^{m+M}\left(\frac{d}{dz}\int_{m}^{y}f_{X,Y}(x,y)dx+\frac{d}{dz}\int_{0}^{z}f_{X,Y}(t+y,y)dt\right)dy\\
        &=\int_{m}^{m+M}f_{X,Y}(z+y,y)dy\\
        &\le\int_{m}^{m+M}\kappa dy\\
        &= M\kappa.
    \end{align*}
    \noindent Finally, (iv) follows from simple change of variable manipulations (e.g. Theorem 22 of \cite{dick2020semi}).
\end{proof}

\noindent{\bf Theorem \ref{thm:en-regression-dispersion} (restated).} 
{\it Assume that the predicted variable and all feature values are bounded by an absolute constant $R$, i.e. $\max\{||X^{(i)}||_{\infty,\infty},||y^{(i)}||_\infty,||\Xv^{(i)}||_{\infty,\infty},||\yv^{(i)}||_\infty\}\le R$. Suppose the predicted variables $y^{(i)}$ in the training set are drawn from a joint $\kappa$-bounded distribution. Let $l_1,\dots, l_T:(0,\lambda_{\max} )^2\rightarrow\R_{\ge 0}$ denote an independent sequence of losses (e.g. fresh randomness is used to generate the validation set features in each round) as a function of the ElasticNet regularization parameter $\lambda=(\lambda_1,\lambda_2)$, $l_i(\lambda)=l_r(\hat{\beta}^{(X^{(i)},y^{(i)})}_{\lambda,f_{EN}},(\Xv^{(i)},\yv^{(i)}))$. The sequence of functions is $\frac{1}{2}$-dispersed, and there is an online algorithm with $\Tilde{O}(\sqrt{T})$ expected regret. The result also holds for loss functions adjusted by information criteria AIC and BIC.}

\begin{proof}
We start with the piecewise-decomposable characterization of the dual class function in Theorem \ref{lem:en-structure}. On any fixed problem instance $P\in\Pi_{m,n}$, as the parameter $\lambda$ is varied in the loss function $\ell_{\text{EN}}(\cdot,P)$ of ElasticNet trained with regularization parameter $\lambda=(\lambda_1,\lambda_2)$, we have the following piecewise structure. There are $k=p3^p$ boundary functions $g_1,\dots,g_k$ for which the transition boundaries are algebraic curves $r_i(\lambda_1,\lambda_2)$, where $r_i$ is a polynomial with degree 1 in $\lambda_1$ and at most $p$ in $\lambda_2$.
Also the piece function $f_\mathbf{b}$ for each sign pattern $\mathbf{b} \in \{0, 1\}^k$ is a rational polynomial function $\frac{q_1^b(\lambda_1,\lambda_2)}{q_2^b(\lambda_2)}$, where $q_1^\b,q_2^\b$ have degrees at most $2p$, and corresponds to a fixed signed equicorrelation set $(\cE,s)$. To show online learnability, we will examine this piecewise structure more closely -- in particular analyse how the structure varies when the predicted variable is drawn from a smooth distribution.

In order to show dispersion for the loss functions $\{l_i(\lambda)\}$, we will use the recipe of \cite{dick2020semi} and bound the worst rate of discontinuities  between any pair of points $\lambda=(\lambda_1,\lambda_2)$ and $\lambda'=(\lambda'_1,\lambda'_2)$ with $||\lambda-\lambda'||_2\le \epsilon$ along the axis-aligned path $\lambda\rightarrow (\lambda'_1,\lambda_2) \rightarrow\lambda'$. First observe that the only possible points at which $l_i(\lambda)$ may be discontinuous are
\begin{itemize}
    \item[(a)] $(\lambda_1,\lambda_2)$ such that $r_i(\lambda_1,\lambda_2)=0$ corresponding to some boundary function $g_i$.
    \item[(b)] $(\lambda_1,\lambda_2)$ such that $q_2^\b(\lambda_2)=0$ corresponding to some piece function $f_\b$.
\end{itemize}
Fortunately the discontinuity of type (b) does not occur for $\lambda_2>0$. From the ElasticNet characterization in Lemma \ref{lem:en-equicorrelation}, and using Lemma \ref{lem:gram-inv}, we know that $q_2(\lambda_2)=\Pi_{j\in[|\cE|]}(\Lambda_{j}+\lambda_2)$, where $(\Lambda_{j})_{j\in[|\cE|]}$ are non-negative eigenvalues of the positive semi-definite matrix ${X^{(i)}_\cE}^TX^{(i)}_\cE$. It follows that $q_2^\b$ does not have positive zeros (for any sign vector $\b$).

Therefore it suffices to locate boundaries of type (a). To this end, we have two subtypes corresponding to a variable entering or leaving the equicorrelation set. 

{\it Addition of $j\notin\cE$.} As observed in the proof of Theorem \ref{lem:en-structure}, a variate $j\notin\cE$ can enter the equicorrelation set $\cE$ only for $(\lambda_1,\lambda_2)$ satisfying 

\[\lambda_1=\frac{\x_j^TX_\cE({X_{\cE}}^TX_\cE+\lambda_2 I_{|\cE|})^{-1}{X_{\cE}}^Ty-\x_j^Ty}{\x_j^TX_\cE({X_{\cE}}^TX_\cE+\lambda_2 I_{|\cE|})^{-1} s\pm1}.\] 
For fixed $\lambda_2$, the distribution of $\lambda_1$ at which the discontinuity occurs for insertion of $j$ is $K_1\kappa$-bounded (by Lemma \ref{lem:bounded}) for some constant $K_1$ that only depends on $R,m,p$ and $\lambda_{\max}$. This implies an upper bound of $K_1\kappa\epsilon$ on the expected number of discontinuities corresponding to $j$ along the segment $\lambda\rightarrow(\lambda'_1,\lambda_2)$ for any $j,\cE$. 

For constant $\lambda_1$, we can use Lemma \ref{lem:gram-inv} and a standard change of variable argument (e.g. Theorem 22 of \cite{dick2020semi}) to conclude that the discontinuties lie at the roots of a random polynomial in $\lambda_2$ of degree $|\cE|$, leading coefficient $1$, and bounded random coefficients with $K_2\kappa$-bounded density for some constant $K_2$ (that only depends on $R,m,p$ and $\lambda_{\max}$). By Theorem \ref{thm:poly-roots}, the expected number of discontinuities along the segment $(\lambda'_1,\lambda_2)\rightarrow\lambda'$ is upper bounded by $K_2K_p\kappa\epsilon$ ($K_p$ only depends on $p$). This implies that the expected number of Lipschitz violations between $\lambda$ and $\lambda'$ along the axis aligned path is $\Tilde{O}(\kappa\epsilon)$ and completes the first step of the recipe in this case ($\Tilde{O}$ notation suppresses terms in $R,m,p$ and $\lambda_{\max}$ as constants). 

{\it Removal of $j'\in\cE$.} The second case, when a variate  $j'\in\cE$ leaves the equicorrelation set $\cE$ for $(\lambda_1,\lambda_2)$ satisfying $$\lambda_1(({X_{\cE}}^TX_\cE+\lambda_2 I_{|\cE|})^{-1}s)_{j'} = (({X_{\cE}}^TX_\cE+\lambda_2 I_{|\cE|})^{-1}{X_{\cE}}^Ty)_{j'},$$ also yields the same bound using the above arguments. Putting together, and noting that we have at most $p3^p$ distinct curves each with $\Tilde{O}(\kappa\epsilon)$ expected number of intersections with the axis aligned path $\lambda\rightarrow\lambda'$, the total expected number of discontinuities is also $\Tilde{O}(\kappa\epsilon)$. This completes the first step (S1) of the above recipe.  

We use Theorem 9 of \cite{dick2020semi} to complete the second step of the recipe, which employs a VC-dimension argument for $K'$ algebraic curves of bounded degrees (here degree is at most $p+1$) to conclude that the expected worst number of  discontinuties along any axis-aligned path between any pair of points $\le\epsilon$ apart is at most $\Tilde{O}(\epsilon T)+O(\sqrt{T\log K'T})$. $K'\le p3^p$ as shown above. This implies that the sequence of loss functions is $\frac{1}{2}$-dispersed, and further there is an algorithm (Algorithm 4 of \cite{balcan2018dispersion}) that achieves $\Tilde{O}(\sqrt{T})$ expected regret.

Finally note that loss functions with AIC and BIC have the same dual class piecewise structure, and therefore the above analysis applies. The only difference is that the value of the piece functions $f_\b$ are changed by a constant (in $\lambda$), $K_{m,p}\le p\log m$. The piece boundaries are the same, and are therefore $\frac{1}{2}$-dispersed as above. The range of the loss functions is now $[0,K_{m,p}+1]$, so the same algorithm (Algorithm 4 of \cite{balcan2018dispersion}) again achieves $\Tilde{O}(\sqrt{T})$ expected regret.
\end{proof}

\section{Lemmas and proof details for Section \ref{sec:classification}}\label{app:classification}
We will first extend the structure for the ElasticNet regression loss functions shown in Theorem \ref{lem:en-structure} to the classification setting. The main new challenge is that there are additional discontinuties due to thresholding the loss function needed for binary classification, which intuitively makes the loss more jumpy and discontinuous as a function of the regularization parameters.

\begin{lemma}\label{lem:en-structure-c}
Let $\cL$ be a set of functions $\{l_{\lambda,\tau}:\Pi_{m,p}\rightarrow \R_{\ge 0}\mid \lambda\in\R^+\times\R_{\ge 0},\tau\in\R\}$ that map a regression problem instance $P\in \Pi_{m,p}$ to the validation classification loss $\ell_{\text{EN}}^c(\lambda,P,\tau)$ of ElasticNet trained with regularization parameter $\lambda=(\lambda_1,\lambda_2)$ and threshold parameter $\tau$. The dual class $\cL^*$ is $(\cF,\cG,(m+p)3^p)$-piecewise decomposable, with $\cF = \{f_{c} : \cL \rightarrow \R\}$ consisting of constant functions $f_{c} : l_{\lambda,\tau}\mapsto c$, where $c\in\R_{\ge 0}$, and $\cG=\{g_r:\cL\rightarrow \{0,1\}\}$ consisting of semi-algebraic sets bounded by algebraic varieties $g_r : l_{\lambda,\tau} \mapsto \mathbb{I}\{r(\lambda_1,\lambda_2,\tau)<0\}$, where $r$ is a polynomial of degree 1 in $\lambda_1$ and $\tau$, and at most $p$ in $\lambda_2$.
\end{lemma}
\begin{proof}
By Lemma \ref{lem:en-equicorrelation}, the EN coefficients $\hat{\beta}_{EN}$ are fixed given the signed equicorrelation set $\cE,s$. As in Theorem \ref{lem:en-structure}, we have $\le p3^p$ boundaries $\cG_1$ corresponding to a change in the equicorrelation set, but the value of the loss also changes when a prediction vector coefficient $\mu_j=(\Xv)_j\hat{\beta}_{EN}$ cross the threshold $\tau$. This is given by $(c_1-c_2{\lambda_1})_j=\tau$ where $$c_1=(\Xv)_\cE(X_{\cE}^TX_{\cE}+\lambda_2I_{|\cE|})^{-1}X_{\cE}^Ty,$$ and $$c_2= (\Xv)_\cE(X_{\cE}^TX_{\cE}+\lambda_2I_{|\cE|})^{-1}s.$$ Therefore, $\hat{\mu}_j=0$ corresponds to the 0-set of $(c_1-c_2{\lambda_1})_j-\tau$. By an application of Lemma \ref{lem:gram-inv}, this is an algebraic variety with degree at most $|\cE|$ in $\lambda_2$ and degree $1$ in $\lambda_1$ and $\tau$. There are at most $m3^p$ such boundary functions $\cG_2$, corresponding to all possibilities of $j,\cE,s$. For a point $(\lambda_1,\lambda_2,\tau)$ with a fixed sign pattern of boundary functions in $\cG_1\cup \cG_2$, the EN coefficients are fixed and also all the predictions on the validation set are fixed. The classification loss $$l_c(\hat{\beta}_{\lambda,f}, (X,y), \tau) = \frac{1}{m}\sum_{i=1}^m(y_i-\sgn(\langle X_i, \hat{\beta}_{\lambda,f}\rangle-\tau))^2$$ is therefore constant in each piece.
Applying Theorem \ref{lem:en-structure} shows the claimed stucture for the ElasticNet classification (dual class) loss function.
\end{proof}

\noindent The above piecewise decomposable structure is helpful in bounding the pseudodimension for the ElasticNet based classifier. For the special cases of Ridge and LASSO we obtain the pseudodimension bounds from first principles.

\noindent{\bf Theorem \ref{thm:ridge-pdim} (restated).}
 {\it Let $\Hr^c$, $\Hl^c$ and $\Hen^c$ denote the set of loss functions for classification problems with at most $m$ examples and $p$ features, with Ridge, LASSO and ElasticNet regularization respectively.
 \begin{itemize}[leftmargin=0.8cm,nosep]
    \item[(i)] $\pdim(\Hr^c)=O(\log mp)$
    \item[(ii)]  $\pdim(\Hl^c)=O(p\log m)$. Further, in the overparameterized regime ($p\gg m$), we have that $\pdim(\Hl^c) = O(m\log \frac{p}{m})$.
    \item[(iii)] $\pdim(\Hen^c)=O(p^2+p\log m)$.
\end{itemize}
}

\begin{proof}
    \item[$\;\;(i)$] For Ridge regression, the estimator $\hat{\beta}_{\lambda,f_2}$ on the dataset $(X^{(i)},y^{(i)})$ is given by the following closed form
\begin{align*}
    \hat{\beta}_{\lambda,f_2} = ({X^{(i)}}^TX^{(i)}+\lambda I_{p_i})^{-1}{X^{(i)}}^Ty^{(i)},
\end{align*}
where $I_{p_i}$ is the ${p_i}\times {p_i}$ identity matrix. By Lemma \ref{lem:gram-inv} each coefficient $(\hat{\beta}_{\lambda,f_2})_k$ of the estimator $\hat{\beta}_{\lambda,f_2}$ is a rational polynomial in $\lambda$ of the form $P_k(\lambda)/Q(\lambda)$, where $P_k,Q$ are polynomials of degrees at most $p_i-1$ and $p_i$ respectively. Thus the prediction on any example $(\Xv^{(i)})_j$ in the validation set $(\Xv^{(i)},\yv^{(i)})$  of any problem instance $P^{(i)}$ can change at most $p_i\le p$ times as $\lambda$ is varied. Recall there are $m_i'\le m$ examples in any validation set. This implies we have at most $mnp$ distinct values of the loss function over the $n$ problem instances. The pseudo-dimension $n$ therefore satisfies $2^n\le mnp$, or $n=O(\log mp)$.
    \item[$\;\;(ii)$] Prior work \cite{efron2004least} shows that the optimal vector $\hat{\beta}\in\R^p$ evolves piecewise linearly with $\lambda$, i.e. $\exists \lambda^{(0)} = 0 < \lambda^{(1)} < \dots < \lambda^{(q)} = \infty$ and $\gamma_0,\gamma_1,\dots,\gamma_{q-1} \in \R^p$
such that $$\hat{\beta}_{\lambda, f_1} =
\hat{\beta}_{\lambda^{(k)},f_1} + (\lambda - \lambda^{(k)})\gamma_k$$ for $\lambda^{(k)}\le \lambda \le\lambda^{(k+1)}$. Each piece corresponds to the addition or removal of at least one of $p$ coordinates to the {\it active set} of covariates with maximum correlation. For any data point $x_j$, $1\le j\le m$, and any piece $[\lambda^{(k)},\lambda^{(k+1)})$ we have that $x_j\hat{\beta}$ is monotonic since $\hat{\beta}$ varies along a fixed vector $\gamma_k$, and therefore can have at most one value of $\lambda$ where the predicted label $\hat{y}$ changes. This gives an upper bound of $mq$ on the total number of discontinuities on any single problem instance $(X^{(i)},y^{(i)},\Xv^{(i)},\yv^{(i)})$, where $q$ is the number of pieces in the solution path. By Lemma 6 of \cite{tibshirani2013lasso}, we have the number pieces in the solution path $q\le 3^p$. Also for the overparameterized regime $p\gg m$, we have the property that there are at most $m-1$ variables in the active set for the entire sequence of solution paths (Section 7, \cite{efron2004least}). Thus, we have that $q\le m{p \choose m-1}\le (\frac{ep}{m})^m$ in this case.

Over $n$ problem instances, the pseudo-dimension satisfies $2^n\le mqn$, or $n=O(\log mq)$. Substituting the above inequalities for $q$ completes the proof.
    \item[$\;\;(iii)$] The proof of Theorem follows the same arguments as Theorem \ref{thm:pdim-en}, using Lemma \ref{lem:en-structure-c} instead of Theorem \ref{lem:en-structure}. By Lemma \ref{lem:en-structure-c}, the dual class $\Hen^*$ of $\Hen$ is $(\cF,\cG,(p+m)3^p)$-piecewise decomposable, with $\cF = \{f_{c} : \cL \rightarrow \R\}$ consisting of constant functions $f_{c} : l_\lambda\mapsto c$, where $c\in\R_{\ge0}$, and $\cG=\{g_r:\cL\rightarrow \{0,1\}\}$ consisting of polynomial thresholds $g_r : u_\lambda \mapsto \mathbb{I}\{r(\lambda_1,\lambda_2,\tau)<0\}$, where $r$ is a polynomial of degree 1 in $\lambda_1$ and $\tau$, and at most $p$ in $\lambda_2$.

Now it is easy to see that $\pdim(\cF^*)=O(1)$ (in particular, a consequence of Lemma \ref{lem:f-pdim}), and by Lemma \ref{lem:g-vcdim} the VC dimension of the dual boundary class is $d_{\cG^*}=O(p)$. A straightforward application of Theorem \ref{thm:pdim-dual} yields \[\pdim(\cH)=O(p\log p+p\log((p+m)3^p))=O(p^2+p\log m).\]
    
    \end{proof}

\noindent We will now restate  and prove Theorem \ref{thm:ridge-dispersion}. This implies that under smoothness assumptions on the data distribution we can learn the data-dependent optimal regularization parameter in the online setting. 

\noindent{\bf Theorem \ref{thm:ridge-dispersion} (restated).} 
{\it Suppose Assumptions \ref{ass:boundedness} and \ref{ass:smooth-features} hold. Let $l_1,\dots, l_T:(0,H]^{d}\times[-H,H]\rightarrow\R$ denote an independent sequence of losses as a function of the regularization parameter $\lambda$, $l_i(\lambda,\tau)=l_c(\hat{\beta}_{\lambda,f},(X^{(i)},y^{(i)}),\tau)$. The sequence of functions is $\frac{1}{2}$-dispersed, and there is an online algorithm with $\Tilde{O}(\sqrt{T})$ expected regret, if $f$ is given by\begin{itemize}[leftmargin=0.8cm,nosep]
    \item[(i)] $f=f_1$ (LASSO),
    \item[(ii)] $f=f_2$ (Ridge), or
    \item[(iii)] $f=f_{EN}$ (ElasticNet).
\end{itemize}}

\begin{proof}
   \item[$\;\;(i)$]
   On any dataset $(X,y)$, the predictions are given by the coefficients of the prediction vector $\hat{\mu}=X(X^TX+\lambda I_p)^{-1}X^Ty$.
Note that by Lemma \ref{lem:gram-inv} $(X^TX+\lambda I_p)^{-1}$, and therefore  $\Xv(X^TX+\lambda I_p)^{-1}X^Ty$, has each element of the form $P_{j}(\lambda)/Q(\lambda)$ with degree of each $P_{j}$ at most $p-1$ and degree of $Q$ at most $p$.
Further, for a fixed $\tau$, by using Lemma \ref{lem:gram-inv-coeffs} and a change of variables, we have that $\hat{\mu}_j=\tau$ is polynomial equation in $\lambda$ with degree $p$ with bounded coefficients that have $K\kappa$-bounded density for some constant $K$ (that depends on $R$, $H$, $m$ and $p$, but not on $\kappa$) and leading coefficient 1. Further, for fixed $\lambda\in(0,H]$, by Lemma \ref{lem:bounded}, the discontinuities over $\tau$ again have $\Tilde{O}(\kappa)$-bounded density. This completes step S1 of the recipe from \cite{dick2020semi}, and Theorem \ref{thm:poly-roots} gives a bound on the expected number of discontinuities in any interval $I$ (over $\lambda$).

To complete step S2, note that the loss function $l_i$ on any instance has at most $pm$ discontinuities, since each coefficient $\hat{\mu}_j$ of the prediction vector can change sign at most $p$ times as $\lambda$ is varied. This implies the VC-dimension argument (Theorem \ref{thm:VC-bound}) applies and the expected maximum number of discontinuities in any interval of width $\epsilon$ is $O(\epsilon T)+O(\sqrt{T\log(mpT)})$, which is $\Tilde{O}(\epsilon T)$ for $\epsilon\ge 1/\sqrt{T}$. Thus, using the recipe from \cite{dick2020semi}, we have shown that the sequence of loss functions is $\frac{1}{2}$-dispersed. This further implies that Algorithm \ref{alg:ddreg}, which implements the Continuous Exp-Weights algorithm of \cite{balcan2018dispersion} for setting the regularization parameter, achieves $\Tilde{O}(\sqrt{T})$ expected regret (\cite{balcan2018dispersion}, Theorem 1).

\item[$\;\;(ii)$]
Since the data-distribution is in particular assumed to be continuous, by Lemma 4 of \cite{tibshirani2013lasso} we know that the LASSO solutions are unique with probability 1. Moreover if $\cE\subseteq[p]$ denotes the equicorrelation set of variables (i.e. covariates with the maximum absolute value of correlation), and $s\in\{-1,1\}^{|\cE|}$ the sign vector (i.e. the sign of the correlations of the covariates in $\cE$), then the LASSO prediction vector $\hat{\mu}=\Xv\hat{\beta}$ is a linear function of regularization parameter $\lambda$ given by
\[\hat{\mu} = c_1-c_2\lambda,\]
where $c_1=(\Xv)_{\cE}(X_{\cE}^TX_{\cE})^{-1}X_{\cE}^Ty$ and $c_2=(\Xv)_{\cE}(X_{\cE}^TX_{\cE})^{-1}s$. Thus for any fixed $\cE,s$ (corresponding to a unique piece in the solution path for LARS-LASSO), we have at most one discontinuity corresponding to $\hat{\mu}_j=\tau$, and the location of this discontinuity has a $K\kappa$-bounded distribution (for constant $K$ independent of $\kappa$) by an application of Lemma \ref{lem:bounded}. Thus, the probability that this discontinuity is located along some axis aligned path $I$ of length $\epsilon$ is at most $K\kappa\epsilon$. A union bound over $j\in[m]$, and over $3^p$ choices of $\cE,s$ (for example, Lemma 6 in \cite{tibshirani2013lasso}) gives the probability of a discontinuity along $I$ is at most $m3^pK\kappa\epsilon$. This completes step S1 of the recipe above.

Now each loss function $l_i$ has at most $m3^p$ discontinuities, and therefore by a VC-dimension argument (Theorem 9 of \cite{dick2020semi}), the expected maximum number of discontinuities along any axis-aligned path of total length $\epsilon$ is $\Tilde{O}(\epsilon T)+O(\sqrt{T(p+\log(mT))})$, which is $\Tilde{O}(\epsilon T)$ for $\epsilon\ge 1/\sqrt{T}$. This completes step S2 of the recipe from \cite{dick2020semi}, and we have shown that the sequence of loss functions is $\frac{1}{2}$-dispersed. As in Theorem \ref{thm:ridge-dispersion}, this implies that Algorithm \ref{alg:ddreg} achieves $\Tilde{O}(\sqrt{T})$ expected regret (\cite{balcan2018dispersion}, Theorem 1).

While we use the worst case bound on the number of solution paths here, algorithmically we can use LARS-LASSO on the given dataset, which is much faster in practice and the running time scales linearly with the actual number of solution paths $q$ (typically $q\ll 3^p$).

\item[$\;\;(iii)$]

The proof uses the piecewise decomposable structure proved in Lemma \ref{lem:en-structure-c}, and establishes dispersion using joint smoothness of $\Xv^{(i)}$ instead of $y^{(i)}$ (as in the proof of Theorem \ref{thm:ridge-dispersion} (i)). The recipe from \cite{dick2020semi} can be used along a 3D axis-aligned path from $(\lambda_1,\lambda_2,\tau)\rightarrow(\lambda_1',\lambda_2',\tau')$. Lemma \ref{lem:bounded} may be used to show bounded density of discontinuities along $\tau$ (keeping $\lambda_1,\lambda_2$ fixed). To complete step S2 of the recipe we can use Theorem 7 from \cite{balcan2021data}.
The arguments are otherwise very similar to the proof of Theorem \ref{thm:en-regression-dispersion}, and are omitted for brevity.
\end{proof}

\end{document}